\documentclass{article}

\PassOptionsToPackage{table,usenames,dvipsnames}{xcolor}

\usepackage[nonatbib,preprint]{neurips_2024}

\usepackage[utf8]{inputenc} 
\usepackage[T1]{fontenc}    
\usepackage{hyperref}       
\usepackage{url}            
\usepackage{booktabs}       
\usepackage{amsfonts}       
\usepackage{nicefrac}       
\usepackage{microtype}      
\usepackage{xcolor}         

\title{2-Cats: 2d Copula Approximating Transforms}

\author{%
  Flavio Figueiredo$^{1}$, \, José G. Fernandes$^{1}$, \, Jackson N. Silva \, and Renato Assunção$^{1,2}$
  \\
  $^{1}\,$Universidade Federal de Minas Gerais\\
  $^{2}\,$ESRI Inc. \\
  Reproducibility: \url{https://anonymous.4open.science/r/2cats-E765/} \\
  Contact Author: \texttt{flavio@dcc.ufmg.br} \\
}

\usepackage{adjustbox}
\usepackage{amsfonts}
\usepackage{amsmath}
\usepackage{amsthm}
\usepackage{amssymb}
\usepackage{booktabs}
\usepackage{bm}
\usepackage{cite}
\usepackage{enumitem}
\usepackage{flushend}
\usepackage{multirow}
\usepackage{rotating}
\usepackage{siunitx}
\usepackage{url}
\usepackage[table,usenames,dvipsnames]{xcolor}
\usepackage{pifont}

\begin{document}

\maketitle

\begin{abstract}
Copulas are powerful statistical tools for capturing dependencies across data dimensions. Applying Copulas involves estimating independent marginals, a straightforward task, followed by the much more challenging task of determining a single copulating function, $C$, that links these marginals. For bivariate data, a copula takes the form of a two-increasing function $C: (u,v)\in \mathbb{I}^2 \rightarrow \mathbb{I}$, where $\mathbb{I} = [0, 1]$. This paper proposes 2-Cats, a Neural Network (NN) model that learns two-dimensional Copulas without relying on specific Copula families (e.g., Archimedean). Furthermore, via both theoretical properties of the model and a Lagrangian training approach, we show that 2-Cats meets the desiderata of Copula properties. Moreover, inspired by the literature on Physics-Informed Neural Networks and Sobolev Training, we further extend our training strategy to learn not only the output of a Copula but also its derivatives. Our proposed method exhibits superior performance compared to the state-of-the-art across various datasets while respecting (provably for most and approximately for a single other) properties of C.
\end{abstract}

\section{Introduction}
\label{s:intro}
Modeling univariate data is relatively straightforward for several reasons. Firstly, a wide range of probability distributions exist that are suitable for different data types, including Gaussian, Beta, Log-Normal, Gamma, and Poisson. This set is expanded by mixture models that use the elemental distributions as components and can model multimodality and other behavior. Secondly, visual tools such as histograms or Kernel Density Estimators (KDE) provide valuable assistance in selecting the appropriate distribution. Lastly, univariate models typically involve few parameters and may have exact solutions, simplifying the modeling process.


However, the process becomes more challenging when modeling multivariate data and obtaining a joint probability distribution. Firstly, only a few classes of multivariate probability distributions are available, with the primary exceptions being the multivariate Gaussian, Elliptical, and Dirichlet distributions. Secondly, simultaneously identifying the dependencies via conditional distributions based solely on empirical data is highly complex.

In 1959, Abe Sklar formalized the idea of Copulas~\cite{Sklar1959,Sklar1996}. This statistical tool allows us to model a multivariate random variable of dimension $d$ by learning its cumulative multivariate distribution function (CDF) using $d$ independent marginal CDFs and a single additional Copulating function $C$. A bivariate example may separately identify one marginal as a Log-Normal distribution and the other as a Beta distribution. Subsequently, the Copula model links these two marginal distributions.

However, the Copula approach to modeling multivariate data faced a significant limitation until recently. Our model choices were confined to only a handful of closed forms for the $C$ Copulas, such as the Gaussian, Frank, or Clayton Copulas~\cite{Czado2022, Grosser2022}. 
Unfortunately, this approach using closed-form Copula models proved inadequate for accurately capturing the complex dependencies in real data. The limited parameterization of these Copulas prevented them from fully representing the intricate relationships between variables.

This situation has changed recently. 
Neural Networks (NNs) are well-known for their universal ability to approximate any function. Consequently, researchers have started exploring NNs as alternatives to closed forms for $C$~\cite{Chilinski2020,Hirt2019,Tagasovska2023,Janke2021,Ling2020,
Ng2021}. However, a limitation of NNs is that they neglect the importance of maintaining the fundamental mathematical properties of their domain of study. 

For Copulas, only three properties define $C$~\cite{Nelsen2006}:
{\bf P1:} $C: (u,v)\in \mathbb{I}^2 \mapsto \mathbb{I}$, where $\mathbb{I} = [0, 1]$; 
{\bf P2:} For any $0 \leq u_1 < u_2 \leq 1$ and $0 \leq v_1 < v_2 \leq 1$, we have that the volume of $C$ is: $V_C(u_1, u_2, v_1, v_2) \equiv C(u_2,v_2)-C(u_2,v_1)-C(u_1,v_2)+C(u_1,y_1) \geq 0$; 
and, {\bf P3:} $C$ is grounded. That is, $C(0, v) = 0$, $C(v, 0) = 0$. Moreover, $C(1, v) = v$ and $C(u, 1) = u$.

This paper proposes 2-Cats (2D Copula Approximating Transforms) as a NN approach that meets the above desiderata. How the model meets P1 to P3 is discussed in Section~\ref{sec:propsandproofs}. {\em P1 and P2 are natural consequences of the model, whereas P3 is guaranteed via Lagragian Optimization}.

To understand our approach, let $H_{\theta}(u, v)$ represent our model (or Hypothesis). Let $G: \mathbb{R}^2 \mapsto \mathbb{I}$ be {\bf any} bivariate CDF with $\mathbb{R}^2$ as support. 
We connect $H_{\bm{\theta}}$ with $G$ as follows: 
\newcommand\numberthis{\addtocounter{equation}{1}\tag{\theequation}}
\begin{align*}
H_{\bm{\theta}}(u, v) &= G\Big(z\big(t_v(u)\big), z(t_u(v)\big)\Big); \,\, z(x) = \log(\frac{x}{1-x});\\
t_v(u) &= \frac{\int_{0}^{u} m_{\bm{\theta}}(x, v)\, dx}{\int_{0}^{1} m_{\bm{\theta}}(x', v) \,dx'}; \,\,
t_u(v) = \frac{\int_{0}^{v} m_{\bm{\theta}}(u, y) \,dy}{\int_{0}^{1} m_{\bm{\theta}}(u, y') \,dy'}; \numberthis \label{eqn-model}
\end{align*}
Here, $m_{\bm{\theta}}(u, v)$ is any NN that outputs a positive number and $z(x)$ is the Logit function that is both monotonic and maps $[0, 1]$ to $[-\infty, \infty]$, i.e., $z: \mathbb{I} \mapsto \mathbb{R}$. Here,   
the function $t_v(u)$ monotonically increases concerning $u$ because of the positive NN output; similarly, $t_u(v)$ is monotonic on $v$. 
The monotonic properties of all transformations ensure that our output $H_{\bm{\theta}}(u, v)$ is monotonic for both $u$ and $v$. Let $z_u = z(t_v(u))$ and $z_v = z(t_u(v))$ to simplify notation and $(U,V)$ a random vector with distribution function $H_{\bm{\theta}}(u, v)$. Then, the probability transform below is valid~\cite[section 4.3]{Casella2001}:
\begin{align*}
H_{\bm{\theta}}(u, v) &= \Pr\big[ Z_u \leq z_u, Z_v \leq z_v \big] 
 = \Pr\big[ t^{-1}_v(U) \leq t^{-1}_v(u), t^{-1}_v(V) \leq t^{-1}_v(v) \big]                  \\
    &=  \Pr\big[ U \leq u, V \leq v \big]  \numberthis \label{eqn-transform}
\end{align*}

To ensure that our model derivatives also approximate the derivatives of $C$, a second contribution of our work is that 2-Cats models are trained similarly to Physics Informed Neural Networks (PINNs)~\cite{Karniadakis2021,Lu2021} by employing Sobolev training~\cite{Czarnecki2017} approximate Copula derivatives. 

There have been several recent NN~\cite{Chilinski2020,Ling2020,Ng2021} and non-NN models~\cite{Nagler2017,Bakam2023,Mukhopadhyay2020} Copula-like models. 
We compare our 2-Cats model with these alternatives, and our results show that the 2-Cats performance (in negative log-likelihood) is better or statistically tied to baselines. Our contributions are:  
\begin{itemize}[itemsep=0pt,topsep=0pt]
    \item We introduce the 2-Cats model, a novel NN Copula approximation. Unlike other NN approaches, we focus on satisfying (either provably or via constraints) the essential requirements of a Copula. We also demonstrate the empirical accuracy of 2-Cats;
    \item We are the first to apply Sobolev training~\cite{Czarnecki2017} in NN-based Copula methodologies. Our approach involves the introduction of a simple empirical estimator for the first derivative of a Copula, which is seamlessly integrated into our training framework.
\end{itemize}

\section{Related Work}
\label{s:related}

After being introduced by Sklar~\cite{Sklar1959,Sklar1996}, Copulas have emerged as valuable tools across various domains~\cite{Grosser2022,Nelsen2006}, including engineering~\cite{Salvadori2004}, geosciences~\cite{Modiri2020,Yang2019}, and finance~\cite{Cherubini2004,Salvadori2004}. Recently, Copulas have gained attention from the ML community, particularly with the advent of DL methods for Copula estimation and their utilization in generative models~\cite{Chilinski2020,Hirt2019,Tagasovska2023,Ling2020,Janke2021,Ng2021}.

Our proposed approach aligns closely with the emerging literature in this domain but makes significant advancements on two fronts. This paper primarily focuses on developing an NN-based Copula model that adheres to the fundamental properties of a Copula. Two methods sharing similar objectives are discussed in \cite{Ling2020,Ng2021} and \cite{Chilinski2020}. However, our approach diverges from these prior works by not confining our method to Archimedean Copulas, as seen in~\cite{Ling2020} and \cite{Ng2021}.

As Chilinski and Silva~\cite{Chilinski2020}, we begin by estimating densities through CDF estimations followed by automatic differentiation to acquire densities. However, our primary difference lies in our emphasis on relating 2-Cats to a Copula, an aspect not prioritized by the authors of that study. 

To assess the effectiveness of our approach, we conduct comparative analyses against previous proposals, including the method presented by Janke et al.~\cite{Janke2021}, proposed as a generative model. Furthermore, 2-Cats does not presume a fixed density function for copula estimation; instead, it is guided by the neural network (NN). Hence, our work shares similarities with non-parametric copula estimators used as baselines\cite{Nagler2017,Mukhopadhyay2020,Bakam2023}. Ultimately, all these methods function as baseline approaches against which we evaluate and demonstrate the superiority of our proposed approach.

For training our models, we utilize Sobolev training as a regularizer. Sobolev training is tailored to estimate neural networks (NNs) by incorporating losses that not only address errors toward the primary objective but also consider errors toward their derivatives. This approach ensures a more comprehensive and accurate training of the NNs. A similar concept is explored in Physics Informed Neural Networks (PINNs)~\cite{Karniadakis2021,Kobyzev2020}, where NNs incorporate derivative information from physical models. This approach enables PINNs to model complex physics-related relationships.

In summary, our objective is to develop an NN-based Copula model that respects the mathematical properties of Copulas and benefits from considering derivative information. Consequently, although our focus is on 2D Copulas, the Pair Copula Decomposition (PCC)~\cite{Aas2009,Czado2010} that requires derivative information and is used to model $d$-dimensional data is valid for our models~\cite{Czado2022}. Due to space constraints, we leave the evaluation of such a decomposition for future work. 


Our work also aligns with the literature on monotonic networks~\cite{Sill1997,Daniels2010,Wehenkel2019} and Integral Networks~\cite{Kortvelesy2023,Liu2023}. The concept of using integral transforms to build monotonic models was introduced by~\cite{Wehenkel2019}. 



\section{On the Desiderata of a Copula and 2-Cats Properties} \label{sec:propsandproofs}

We now revisit the discussion on the 2-Cats model as introduced in Section~\ref{s:intro}. 
One of the building blocks of 2-Cats is the NN $m_{\bm{\theta}}(u, v)$. In particular, we require that this NN outputs only positive numbers so that its integral captures a monotonic function. We achieve this by employing Exponential Linear Units (ELU) plus one activation (ELU+1) in every layer.

Let us consider the computation details of $t_v(u)$, with the understanding that the process for $t_u(v)$ is analogous. In our implementation, we adopt Cumulative Trapezoid integration. Initially, we divide the range $[0, 1]$ into $200$ equally spaced intervals. When computing $t_v(u)$, the value of $u$ is inserted into this interval while preserving the order. Subsequently, the $m_{\bm{\theta}}(x, v)$ NN is evaluated for the $201$ corresponding values of parameter $x$. Trapezoidal integration is then applied to calculate the integral value at $u$ (for the numerator) and $1$ (for the denominator). Finally, when $u=0$, $t_v(0) = 0$.

To compute the Jacobian and Hessian of $H_{\bm{\theta}}$, as is done in PINNs~\cite{Karniadakis2021,Lu2021}, we rely on symbolic computation from modern differentiable computing frameworks such as Jax and Pytorch for this case. 
In other words, $\frac{\partial H_{\bm{\theta}}(u, v)}{\partial u}$ and $\frac{\partial H_{\bm{\theta}}(u, v)}{\partial v}$ are estimated via symbolic differentiation via the Jacobian\footnote{\url{https://jax.readthedocs.io/en/latest/_autosummary/jax.jacobian.html}}, while $\frac{\partial^2 H_{\bm{\theta}}(u, v)}{\partial u \partial v}$ is estimated via symbolic differentiation using the Hessian\footnote{\url{https://jax.readthedocs.io/en/latest/_autosummary/jax.hessian.html}}. 


Now, let us present our main properties. Appendix~\ref{appn:samp} and ~\ref{appn:uat} complements these properties.

\newtheorem*{pone}{Theorem P1}
\begin{pone}
$H_{\bm{\theta}}: (u,v)\in \mathbb{I}^2 \mapsto \mathbb{I}$, where $\mathbb{I} = [0, 1]$. 
\end{pone}
\begin{proof}
Remind that $z_u = z\big(t_v(u)\big)$ and $z_v = z\big(t_u(v)\big)$. Also, $(t_v(u), t_u(v))\in \mathbb{I}^2$. These transforms are also able to cover the entire range of $\mathbb{I}^2$ (when $u=0$ we have $t_v(0) = 0$, when $u=1$ we have $t_v(1) = 1$, the same goes for the $t_u(v)$ transform). The function $z$ maps the domain $\mathbb{I}^2$ into $\mathbb{R}^2$. Given that the bivariate CDFs work in this domain: $H_{\bm{\theta}}: (u,v)\in \mathbb{I}^2 \mapsto \mathbb{I}$, where $\mathbb{I} = [0, 1]$.
\end{proof}

\newtheorem*{ptwo}{Theorem P2}
\begin{ptwo}
The 2-Cats copula $H_{\bm{\theta}}(u,v)$ satisfies the non-negative volume property. That is, 
for any $0 \leq u_1 < u_2 \leq 1$ and $0 \leq v_1 < v_2 \leq 1$, we have that  $V_{H_{\bm{\theta}}}(u_1, u_2, v_1, v_2) \equiv H_{\bm{\theta}}(u_2,v_2)-H_{\bm{\theta}}(u_2,v_1)-H_{\bm{\theta}}(u_1,v_2)+H_{\bm{\theta}}(u_1,v_1) \geq 0$. 
\end{ptwo}\renewcommand*{\proofname}{Intuition.}\renewcommand{\qedsymbol}{}
\begin{proof}
This is a straightforward consequence of $G$ being a bivariate cumulative distribution function and the monotonicity of the $z$, $t_u$, and $t_v$ transforms. 
The same transform we explored above for the single variate case is valid for multiple variables. That is, fact that $z_u = z(t_v(u))$ and $z_v = z(t_u(v))$ are monotonic guarantees that $G(z_u, z_v)$ defines a {\em valid transform} Eqn~\eqref{eqn-transform}.

Considering a bivariate CDF, for any given point in its domain (e.g., $t_v(u), t_u(v)$), the CDF represents the accumulated probability up to that point. Let the origin be $(t_v(0), t_u(0))$. As either $t_v(u)$, $t_u(v)$, or both values increase, the CDF values can only increase or remain constant. Therefore, the volume under the CDF surface will always be positive. Now, consider the following Corollary of Eq~\ref{eqn-transform}.
\end{proof}

\renewcommand*{\proofname}{Corollary: The second derivative of 2-Cats is a pseudo-likelihood.}
\begin{proof}
Let $F(x_1, x_2) = \Pr[X_1 \leq x_1, X_2 \leq x_2]$ be the bivariate CDF associated with RVs $X_1$ and $X_2$. With $u = F_{X_1}(x_1)$ and $v = F_{X_2}(x_2)$ being the marginal CDFs, we reach that: $F(x_1, x_2) = \Pr\big[ U \leq u, V \leq v \big]$. By Eq~\eqref{eqn-transform}, $\Pr\big[ U \leq u, V \leq v \big] = H_{\bm{\theta}}(u, v) = H_{\bm{\theta}}(F_{X_1}(x_1), F_{X_2}(x_2))$. 2-Cats {\em is} also capturing the same {\it bivariate} CDF of $C$. The bivariate density, $f(x_1, x_2) = \frac{\partial^2 F(x_1, x_2)}{\partial x_1\,\partial x_2}$, thus is:
\begin{align}f(x_1, x_2) =\frac{\partial^2H_{\bm{\theta}}(F_{X_1}(x_1), F_{X_2}(x_2))}{{\partial x_1\,\partial x_2}} = f_{X_1}(x_1)\,f_{X_2}(x_2) \frac{\partial^2H_{\bm{\theta}}(u, v)}{{\partial u\,\partial v}} \label{ll}
\end{align}
\noindent The above equation is solved via the Chain rule. A consequence of it, is that the likelihood $f(x_1, x_2)$, is proportional to $h_{\bm{\theta}}(u, v) = \frac{\partial^2H_{\bm{\theta}}(u, v)}{{\partial u\,\partial v}}$. This term is known as a pseudo-likelihood~\cite{Genest1995,Genest2007}.
\end{proof}
\renewcommand*{\proofname}{Proof P2.}\renewcommand{\qedsymbol}{$\square$}
\begin{proof}
Given that $f(x_1, x_2) \geq 0$ by definition, and also $f_{X_1}(x_1) \geq 0$ and $f_{X_2}(x_2) \geq 0$: $h_{\bm{\theta}}(u, v) \geq 0$. The volume being the double integral of $h_{\bm{\theta}}(u, v)$ is always positive or zero.
\end{proof}
\newtheorem*{pthree}{Property P3}
\begin{pthree}
$H_{\bm{\theta}}(u, 0) = 0$, $H_{\bm{\theta}}(0, v) = 0$, $H_{\bm{\theta}}(u, 1) \approx u$, and $H_{\bm{\theta}}(1, v) \approx v$. Notice that this property is a relaxation of the one presented in Section 1.
\end{pthree} This last property of a Copula is the one that guarantees that the Copula will have uniform marginals. In particular, the terms: $C(u, 1) = u$ and $C(1, v) = v$ are of utmost importance for sampling. To grasp this, consider some valid Copula $C$. Here, $C(u, 1) = \Pr[U \leq u, V \leq 1] = \Pr[U \leq u] = u$, which is the CDF of the Uniform distribution. A similar argument exists for $v$. 

2-Cats does not provably meet P3 as it does P1 and P2. Nevertheless, we prove that:  $H_{\bm{\theta}}(u, 0) = 0$ and $H_{\bm{\theta}}(0, v) = 0$, while approximating $H_{\bm{\theta}}(u, 1) \approx u$, and $H_{\bm{\theta}}(1, v) \approx v$. 

\newtheorem*{lemmap31}{Lemma P3.1}
\begin{lemmap31}
$H_{\bm{\theta}}(u, 0) = 0$ and $H_{\bm{\theta}}(0, v) = 0$
\end{lemmap31}
\renewcommand*{\proofname}{Proof.}
\begin{proof}
 Consider the value 
of $H_{\bm{\theta}}(u, 0)$ when $v=0$. As $t_u(0)=0$, we have $z(t_u(0)) = z(0) = \lim_{x\rightarrow 0} z(x) = -\infty$ and hence 
$H_{\bm{\theta}}(u, 0) = G(z_u, -\infty)$. Given that $G$ is a bivariate CDF, 
$G(z_u, -\infty) = \lim_{w \rightarrow -\infty} G(z_u, w) = 0$ for any $z_u$. The same proof exists for $H_{\bm{\theta}}(0, v)$.
\end{proof}


\newtheorem*{conj32}{Conjecture P3.2}
\begin{conj32} $H_{\bm{\theta}}(u, 1) \approx u$ and $H_{\bm{\theta}}(1, v) \approx v$.
\end{conj32}\renewcommand*{\proofname}{Meeting P3.2 via Lagrangian.}
\begin{proof}\renewcommand{\qedsymbol}{}
To meet this relaxed property, we propose the following constrained optimization. Let the inputs of our model be comprised of the set $\mathcal{D} = \{u_i, v_i\}$, where $i \in [1, n = |\mathcal{D}|]$. Moreover, let $L_{\bm{\theta}}(D)$ be some loss term used to optimize our model (see the next section). Our optimization for 2-Cats will focus on the following constrained problem:
\begin{align}
\arg\,\min_{\bm{\theta}} L_{\bm{\theta}}(D) \quad
\textrm{s.t.} \quad \forall u \in [0, 1]: H_{\bm{\theta}}(u, 1) - u = 0 \, \text{and} \, \forall v \in [0, 1]: H_{\bm{\theta}}(1, v) - v = 0.
\end{align}

Although it is trivial to understand why this optimization meets P3.2, performing such an optimization presents several challenges. The first challenge is how to model constraints. A natural first choice is to consider the square of constraints. However, squared constraint optimization is an ill-posed optimization problem~\cite[Chapter~2.1]{Bertsekas2014}, and ~\cite[Section~2.1]{Platt1987}. Secondly, our model may correctly estimate that $H_{\bm{\theta}}(u, 1) = u$ (or similarly that $H_{\bm{\theta}}(1, v) = v$), but this will not necessarily ensure that the partial derivatives of $H_{\bm{\theta}}(u, 1)$ and $H_{\bm{\theta}}(1, v)$ are accurate (see Appendix~\ref{appn:deriv}). Such derivatives {\bf must equal one} for these inputs as the Copula marginals are distributed according to an $Uniform(0, 1)$. Thirdly, evaluating these constraints over the entire domain $u \in [0, 1]$ and $v \in [0, 1]$ is impossible. The final challenge is how to optimize constraints. Here, a natural approach is to consider Lagrangian multipliers~\cite{Platt1987,Bertsekas2014,Walsh1975}. However, in such cases, the solution to the optimization problem will lie on a saddle point of the loss surface, and gradient methods (used in NNs and 2-Cats) do not converge on saddle points (Section 3.1 of~\cite{Platt1987} provides a simple intuition on why this is so). 

We tackle the first and second challenges by considering the {\bf derivative} of the square of the above constraints, i.e: $\frac{\partial (H_{\bm{\theta}}(u, 1) - u)^2}{\partial u}$ and $\frac{\partial (H_{\bm{\theta}}(1, v) - v)^2}{\partial v}$, which via the chain rule will be defined as follows:
\begin{align}
r^{1}_{\bm{\theta}}(u) = 2\, \big( H_{\bm{\theta}}(u, 1) - u \big) \big( \frac{\partial H_{\bm{\theta}}(u, 1)}{\partial u} - 1 \big) \, \text{and} \, r^{2}_{\bm{\theta}}(v) = 2\, \big( H_{\bm{\theta}}(1, v) - v \big) \big( \frac{\partial H_{\bm{\theta}}(1, v)}{\partial v} - 1 \big).
\end{align}
$r^{1}_{\bm{\theta}}(u)$ and $r^{2}_{\bm{\theta}}(v)$ stand for requirements on $u$ and $v$, respectively. Both requirements are always positive or zero because $f(x) < 1 \iff \int_0^x f(x) dx < x$. Consequently, both requirements will only equal zero when P3.2 is met (our goal). One advantage of this constraint is that it also considers our model's derivative, which tackles our second challenge above.

We evaluated the constraints on our training set to tackle the third challenge. This leads to the optimization stated below. $r_{\bm{\theta}}(D)$ are the constraints on the training data. From the arguments above, we have that: $r_{\bm{\theta}}(D) \geq 0$ and $r_{\bm{\theta}}(D) = 0$ if and only if the constraints are met.
\begin{align}
\hat{\bm{\theta}} = \arg\,\min_{\bm{\theta}} L_{\bm{\theta}}(D) + r_{\bm{\theta}}(D), \,\,\text{where: }\,\, r_{\bm{\theta}}(D) = \sum_{i=1}^{n}r^{1}_{\bm{\theta}}(u_i) + \sum_{i=1}^{n}r^2_{\bm{\theta}}(v_i). \label{eq:lag}
\end{align}
As in~\cite{Bello2022}, to solve the above problem, we treat $r_{\bm{\theta}}(D)$ as a barrier and optimize as follows:
\begin{enumerate}[topsep=0mm,itemsep=0mm]
    \item Let $\lambda \in (0, 1]$ and $\alpha \in (0, 1)$ be hyper-parameters (e.g., we set $\lambda=1$ and $\alpha=0.95$). Also, let $\bm{\theta}^{(0)}$ be the initial model parameters (initialized as in~\cite{Klambauer2017} -- aka Lecun Normal).
    \item For every training iteration $t \in [1, T]$:
    \begin{enumerate}[topsep=0mm,itemsep=0mm]
    \item $\bm{\theta}^{(t)} = \text{stochatic gradient step} \big(\lambda \, L_{\bm{\theta}^{(t-1)}}(D) + r_{\bm{\theta}^{(t-1)}}(D)\big)$ via ADAM~\cite{Kingma2014}
    \item $\lambda = \alpha\,\lambda$
    \end{enumerate}
    \item Set our model as $\hat{\bm{\theta}} = \bm{\theta}^{(T)}$
\end{enumerate}

As $t \rightarrow \infty$, we have that $\lambda \rightarrow 0$, leading our optimization to focus on the constraints. Even though we may not start our optimization in a feasible solution to $r_{\bm{\theta}^{(t-1)}}(D)$~\cite{Bello2022}, the fact that our loss surface is fully differentiable allows the optimizer to reach such a solution if not stuck on some local minima/saddle. Nevertheless, we align with the {\em empirical evidence} that in overparametrized models, such as deep NNs, local minima are rare (see Chapter 6.7 -- The Optimality of Backpropagation -- of~\cite{Baldi2021}) and that such stochastic optimizers, such as ADAM, are suitable for escaping saddle points~\cite{Baldi2021}. Moreover, our initialization, $\bm{\theta}^{(0)}$, is suitable for satisfying regularizers~\cite{Klambauer2017}. 

For comparison, we shall present results with and without these Lagrangian terms. \end{proof}
\section{Sobolev Losses for 2-Cats Models}
\label{s:training}



We have already discussed how our approach ensures the validity of Copula functions. To train our models, we designed our loss function as a weighted sum, i.e.:
\begin{equation}
L_{\bm{\theta}}(\mathcal{D}) = w^C L^{C}_{\bm{\theta}}(\mathcal{D}) + w^{dC} L^{dC}_{\bm{\theta}}(\mathcal{D}) + w^{c} L^{c}_{\bm{\theta}}(\mathcal{D})
    \label{eq:lossfunction}
\end{equation}
\noindent Here, $w^C$, $w^{dC}$, and $w^c$ are loss weights.

The first component, $L^{C}_{\bm{\theta}}(\mathcal{D})$, stimulates the copula to closely resemble the empirical cumulative distribution function derived from the observed data. In this way, our model learns to capture the essential characteristics of the data. 
The second component, $L^{dC}_{\bm{\theta}}(\mathcal{D})$, imposes penalties on any disparities between the fitted copula's first-order derivatives and the data-based estimates. A Copula's first derivative is essential for sampling (see Appendix~\ref{appn:samp}) and Vine methods. The third component, $L^{c}_{\bm{\theta}}(\mathcal{D})$, focuses on the copula's second-order derivative, which is linked to the probability density function of the data, and evaluates its proximity to the empirical likelihood. Incorporating this aspect in our loss function enhances the model's ability to capture the data's distribution. 

From the arguments above, all three terms play a role in Copula modeling. While obtaining the first component is relatively straightforward, the other two components necessitate some original work. We explain each of them in turn.  Before continuing, we point out that Appendix~\ref{appn} provides an example of how NNs behave when estimating derivatives and integrals of functions. 


Let $F(x_1, x_2) = \Pr\big[ X_1 \leq x_1, X_2 < x_2]$ be the bivariate CDF for some random vector $\mathbf{x}$ and $\widehat{F^n}(\mathbf{x})$ be the empirical cumulative distribution function (ECDF) estimated from a sample of size $n$.
 The Glivenko-Cantelli theorem states that: $\widehat{F^n}(x_1, x_2)\ \xrightarrow{\text{a.s.}}\ F(\mathbf{x})$ as $n \xrightarrow\ \infty$~\cite{Naaman2021}.

We can explore these results to define the relationship between our model, the ECDF, and the true CDF. Being $\hat{\bm{\theta}}$ the estimated parameters and given that the CDF $F(x_1, x_2)$ and the Copula $C(u, v)$ evaluate to the same value (i.e., $u$ and $v$ are the inverse of the marginal CDFs of $x_1$ and $x_2$ respectively):
\begin{align*}
C(u, &v) = F(F_{X_1}^{-1}(u), F_{X_2}^{-1}(v)) = F(x_1, x_2) \approx \widehat{F^n}(x_1, x_2) \approx  H_{\hat{\bm{\theta}}}(u, v) = H_{\hat{\bm{\theta}}}(\widehat{F_{X_1}}(x_1), \widehat{F_{X_2}}(x_2)) .
\end{align*}


Considering these definitions, we can define the loss functions based on the ECDF and the output:
\begin{align}
    L^{C}_{\bm{\theta}}(\mathcal{D}) = \frac{1}{n} \sum_{u_i, v_i \in \mathcal{D}} \big( H_{\bm{\theta}}(u_i, v_i) - \widehat{F^n}(x_{i,1}, x_{i,2}) \big)^2.
\end{align}



Now for our second term, $L^{dC}_{\bm{\theta}}(\mathcal{D})$, a natural step for the first derivatives would be to define a mean squared error (or similar) loss towards these derivatives. The issue is that we need to define empirical estimates from data for both derivatives: $\frac{\partial C(u, v)}{\partial u}$ and $\frac{dC(u, v)}{\partial v}$. To estimate such derivatives, we could have employed conditional density estimators~\cite{Papamakarios2017,Sohn2015,Hall1999}. {\em Nevertheless, methods as~\cite{Papamakarios2017,Sohn2015} are also deep NNs that have the drawback of requiring extra parameters to be estimated. Even classical methods such as~\cite{Hall1999} suffer from this issue.} As a minor contribution, we explore the underlying properties of a Copula to present an {\em empirical approximation} for such derivatives. 

For a 2d Copula, the first derivative of $C$ has the form~\cite{Nelsen2006}:
$\frac{\partial C(u, v)}{\partial u} = \Pr\big[V \leq v \mid U = u \big].$

The issue with estimating this function from data is that we cannot filter our dataset to condition on $U = u$ (we are working with data where $u \in \mathbb{I}$, a real number). However, using Bayes rule we can rewrite this equation. First, let us rewrite the above equation in terms of density functions. That is, if $c_u(v) = F(V \leq v \mid U = u \big) = \Pr\big[V \leq v \mid U = u \big]$ is a cumulative function, $c(v)$ and $c(u)$ are the marginal copula densities (uniform by definition). Using Bayes rule, we have that:
\begin{align}
    \frac{\partial C(u, v)}{\partial u} &= c_u(v) = \frac{c\big(U = u \mid V \leq v \big) F\big( V \leq v \big)}{c\big(u\big)} = v \, c\big(U = u \mid V \leq v \big).
\end{align}
\noindent Where the last identity above comes from the fact that for Copulas, the marginal distributions $U$ and $V$ are uniform, leading to $c(u) = c(v) = 1$. Also, we shall have that $F\big( V \leq v \big) = \Pr\big[V \leq v \big] = v$ and $F\big(U \leq u \big) \Pr\big[U \leq u \big] = u$~\cite{Nelsen2006}. Now, estimate $c\big(U = u \mid V \leq v \big)$.

To do so, let us define $\frac{\widehat{\partial C^n}(u, v)}{\partial u} \approx c\big(U = u \mid V \leq v \big)$ as an empirical estimate of such a density using $n$ data points. We employ the widely used Kernel Density Estimation (KDE) to estimate this function. A fast algorithm for our estimation will work as follows: we arrange our dataset as a table of two columns, where each row contains the pairs of $u, v$ (the columns). For efficiency, we create views of this table sorted on column $u$ or column $v$. When we iterate this table sorted on $v$, finding the points where $V \leq v$ is trivial, as these are simply the previously iterated rows. If we perform KDE on the column $u$ for these points, we shall have an empirical estimate of the density: $c\big(U = u \mid V \leq v \big)$. By simply multiplying this estimate with $v$, we have our empirical estimation of the partial derivative with regards to $u$, that is: $\frac{\widehat{\partial C^n}(u, v)}{\partial u}$. Similarly, we can estimate $\frac{\widehat{\partial C^n}(u, v)}{\partial v}$. Now, our loss term is:
\begin{align}
    L^{dC}_{\bm{\theta}}(\mathcal{D}) = \frac{1}{2n}\sum_{u_i, v_i \in \mathcal{D}}\big(&(\frac{\partial H_{\bm{\theta}}(u_i, v_i)}{\partial u} - \frac{\widehat{\partial C^n}(u_i, v_i)}{\partial u})^2  + (\frac{\partial H_{\bm{\theta}}(u_i, v_i)}{\partial v} - \frac{\widehat{\partial C^n}(u_i, v_i)}{\partial v})^2\big).
\end{align}


Finally, we focus on the last part of our loss, the a pseudo-likelihood:
\begin{align}
L^{c}_{\bm{\theta}}(\mathcal{D}) = -\frac{1}{n} \sum_{u_i, v_i \in \mathcal{D}}  \log(\frac{\partial^2 H_{\bm{\theta}}(u_i, v_i)}{\partial u \partial v}).
\end{align}



It is essential to understand why the three losses are needed. It might seem that simply minimizing this loss would be sufficient. However, we can show that integrals of NN are also approximations up to an asymptotical constant~\cite{Liu2023}, but we still need to ensure that this constant is acceptable (see Appendix~\ref{a:int}). Moreover, Appendix~\ref{appn:ablation} presents an ablation study on the impact of all three losses.





\section{Experimental Results} \label{s:results}

We now present our experimental results on different fronts: (1) validating our Empirical estimates for the first derivative, a crucial step when training 2-Cats; (2) evaluating 2-Cats on synthetic and real datasets without the lagrangian term; and, (3) evaluating the impact of the lagrangian term. 

Before presenting results, we note that our models were implemented using Jax\footnote{\url{https://jax.readthedocs.io/en/latest/}}/Flax\footnote{\url{https://flax.readthedocs.io}}. Given that Jax does not implement CDF methods for Multivariate Normal Distributions, we ported a fast approximation of bivariate CDFs~\cite{Tsay2023} and employed it for our models that rely on Gaussian CDFs/Copulas. Our code is available at {\bf \url{https://anonymous.4open.science/r/2cats-E765/}}. Baseline methods were evaluated using the author-supplied code (details are on Appendix~\ref{appn:code}). 
\subsection{Empirical First Derivative Estimator}

First, we validate our empirical estimations for the first derivative of Copulas. We use the closed-form equations for the first derivatives of Gaussian, Clayton, and Frank Copulas described in ~\cite{Schepsmeier2014}. 


These Copulas have a single parameter, which we varied. Gaussian Copulas ($\rho$) has a correlation parameter, and we set it to: 0.1, 0.5, and 0.9. Clayton and Frank copulas' mixing parameter ($\theta$) was set to 1, 5, and 10. For each of these configurations, we measured the coefficient of determination ($R2$) between empirical estimates and exact formulae. Overall, we found that R2 was above $0.899$ in every setting, validating that our estimates are pretty accurate in practive.

\subsection{2-Cats on Datasets (Without Lagrangian Term)}



We now turn to our main comparisons. Our primary focus, as is the case on baseline methods, will be on capturing the {\bf data likelihood} (Eq~\eqref{ll}). In our experiments, the PDFs for $X_1$ and $X_2$ were estimated via KDE. The bandwidth for these methods is determined using Silverman's rule~\cite{Silverman1986}.


As is commonly done, we evaluate the natural log of this quantity as the data log-likelihood and use its negative to compare 2-Cats with other methods. With this metric of interest in place, we considered two 2-Cats models: (1) 2-Cats-G (Final Layer Gaussian): 2-Cats with a Gaussian bivariate CDF as $G$; and, (2) 2-Cats-L (Final Layer Logistic): 2-Cats with a Logistic bivariate CDF as $G$.

Thus, we considered two forms of $G$. The first was the CDF of a bivariate Gaussian Distribution~\cite{Tsay2023}. The second one is the CDF of the Flexible bivariate Logistic (see Section 11.8 of~\cite{Arnold1992}):
\begin{align*}
    &F(x_1, x_2) = \Pr[X_1 < x_1, X_2 < x_2] = \big((1 + e^{-\alpha \frac{x_1 - \mu_1}{\sigma_1}} + e^{-\alpha \frac{x_2 - \mu_2}{\sigma_2}} + e^{-\alpha (\frac{x_1 - \mu_1}{\sigma_1} + \frac{x_2 - \mu_2}{\sigma_2} ) })^{\frac{1}{\alpha}}\big)^{-1}
\end{align*}
\noindent $\mu_1$, $\mu_2$, $\sigma_1$, $\sigma_2$ and $\alpha$ are free parameters optimized by the model. The same goes for the free parameters of the bivariate Gaussian CDF (the means of each dimension, $\mu_1$ and $\mu_2$, and the correlation parameter $\rho$). Moreover, for each model, we employed four-layer networks with hidden layers having sizes of 128, 64, 32, and 16. We now discuss loss weights.

Recall that the data density is proportional to the pseudo-likelihood. We primarily emphasized the likelihood loss ($w^{c}$). This accentuates our pivotal component. Notably, the scale of its values differs from that of the MSE components, necessitating a proportional adjustment in weight selection. In light of these factors, we fixed $w^{C}=0.01$, $w^{dC}=0.5$, and $w^{c}=0.1$. We note that these hyper-parameters were sufficient to show how our approach is better than the literature. 
Our models were trained using early stopping. Here, for each training epoch, we evaluated the pseudo-log-likelihood {\em on the training set} and kept the weights of the epoch with the best likelihood {\em on the training dataset}. 

 
\begin{table*}[t!]
\centering
\caption{Results for Synthetic Data. Copula parameters are shown on the header. Marginals come from two Normal distributions ($\mu=0, \sigma=1$). $95\%$ confidence intervals (bootstrap) show. Overall, most methods perform statistically the same. {\em This is a positive result; the simulated data comes from well-behaved models. Following the principle of maximum likelihood, it would be unlikely to expect methods to statistically outperform the Parametric (Par) approach}. 
}
\begin{adjustbox}{max width=\textwidth}
\begin{tabular}{llrrrrrrrrr}
\toprule
& & \multicolumn{3}{c}{Gaussian ($\rho$)} & \multicolumn{3}{c}{Clayton ($\theta$)}  & \multicolumn{3}{c}{Frank/Joe ($\theta$)} \\
& &  0.1 & 0.5 & 0.9 & 1 & 5 & 10 & 1 & 5 & 10 \\
\midrule
\multirow{11}{*}{\rotatebox[origin=c]{90}{{\em Non-Deep Learn}}} &
Par & $ 2.91 \pm 0.09 $ & $ 2.69 \pm 0.09 $ & $ 2.09 \pm 0.09 $ & $ 2.65 \pm 0.09 $ & $ 1.89 \pm 0.11 $ & $ 1.39 \pm 0.10 $ & $ 2.78 \pm 0.08 $ & $ 2.57 \pm 0.09 $ & $ 2.10 \pm 0.09 $ \\
&Bern & $ 2.91 \pm 0.09 $ & $ 2.69 \pm 0.09 $ & $ 2.15 \pm 0.09 $ & $ 2.66 \pm 0.08 $ & $ 2.06 \pm 0.09 $ & $ 1.76 \pm 0.08 $ & $ 2.78 \pm 0.08 $ & $ 2.58 \pm 0.09 $ & $ 2.15 \pm 0.08 $ \\
&PBern & $ 2.91 \pm 0.09 $ & $ 2.69 \pm 0.09 $ & $ 2.15 \pm 0.10 $ & $ 2.67 \pm 0.09 $ & $ 2.12 \pm 0.10 $ & $ 1.92 \pm 0.08 $ & $ 2.78 \pm 0.08 $ & $ 2.57 \pm 0.09 $ & $ 2.13 \pm 0.08 $ \\
&PSPL1 & $ 2.91 \pm 0.09 $ & $ 2.69 \pm 0.09 $ & $ 2.15 \pm 0.11 $ & $ 2.67 \pm 0.09 $ & $ 2.14 \pm 0.18 $ & $ 1.61 \pm 0.13 $ & $ 2.78 \pm 0.08 $ & $ 2.57 \pm 0.09 $ & $ 2.10 \pm 0.09 $ \\
&PSPL2 & $ 2.91 \pm 0.09 $ & $ 2.69 \pm 0.09 $ & $ 2.13 \pm 0.10 $ & $ 2.67 \pm 0.09 $ & $ 2.02 \pm 0.11 $ & $ 1.62 \pm 0.10 $ & $ 2.78 \pm 0.08 $ & $ 2.57 \pm 0.09 $ & $ 2.11 \pm 0.09 $ \\
&TTL0 & $ 2.92 \pm 0.09 $ & $ 2.69 \pm 0.09 $ & $ 2.09 \pm 0.09 $ & $ 2.65 \pm 0.08 $ & $ 2.00 \pm 0.21 $ & $ 1.46 \pm 0.09 $ & $ 2.79 \pm 0.08 $ & $ 2.57 \pm 0.09 $ & $ 2.12 \pm 0.09 $ \\
&TLL1 & $ 2.92 \pm 0.09 $ & $ 2.70 \pm 0.10 $ & $ 2.09 \pm 0.09 $ & $ 2.65 \pm 0.09 $ & $ 1.98 \pm 0.21 $ & $ 1.43 \pm 0.10 $ & $ 2.79 \pm 0.08 $ & $ 2.57 \pm 0.09 $ & $ 2.17 \pm 0.20 $ \\
&TLL2 & $ 2.91 \pm 0.09 $ & $ 2.75 \pm 0.20 $ & $ 2.08 \pm 0.09 $ & $ 2.65 \pm 0.09 $ & $ 2.03 \pm 0.25 $ & $ 1.41 \pm 0.10 $ & $ 2.79 \pm 0.09 $ & $ 2.57 \pm 0.09 $ & $ 2.11 \pm 0.09 $ \\
&TLL2nn & $ 2.91 \pm 0.09 $ & $ 2.69 \pm 0.09 $ & $ 2.09 \pm 0.10 $ & $ 2.65 \pm 0.09 $ & $ 1.93 \pm 0.11 $ & $ 1.42 \pm 0.10 $ & $ 2.78 \pm 0.08 $ & $ 2.57 \pm 0.09 $ & $ 2.12 \pm 0.09 $ \\
&MR & $ 2.91 \pm 0.09 $ & $ 2.70 \pm 0.09 $ & $ 2.16 \pm 0.10 $ & $ 2.68 \pm 0.08 $ & $ 2.01 \pm 0.11 $ & $ 1.54 \pm 0.11 $ & $ 2.79 \pm 0.08 $ & $ 2.57 \pm 0.09 $ & $ 2.11 \pm 0.08 $ \\
&Probit & $ 2.91 \pm 0.09 $ & $ 2.69 \pm 0.09 $ & $ 2.11 \pm 0.09 $ & $ 2.66 \pm 0.08 $ & $ 2.05 \pm 0.20 $ & $ 1.50 \pm 0.10 $ & $ 2.78 \pm 0.08 $ & $ 2.57 \pm 0.09 $ & $ 2.12 \pm 0.08 $ \\
\midrule
\multirow{3}{*}{\rotatebox[origin=c]{90}{{\em DL}}}
&NL & $1.46\pm0.08$ & $1.32\pm0.08$ & $ 0.63\pm0.07$ & $1.20\pm0.06$ & $0.47\pm0.09$ & $-0.05\pm0.10$ & $1.39\pm0.07$ & $1.26\pm0.08$ & $0.84\pm0.09$\\
&IGC & $2.92\pm0.10$ & $2.76\pm0.10$ & $2.09\pm0.10$ & $2.64\pm0.08$ & $1.93\pm0.09$ & $1.56\pm0.11$ & $2.87\pm0.09$ & $2.74\pm0.12$ & $2.32\pm0.13$\\
\midrule
\multirow{3}{*}{\rotatebox[origin=c]{90}{{\em Our}}}
&2-Cats-L & $  2.95 \pm 0.09   $ & $  2.79 \pm 0.09   $ & $ 1.91 \pm 0.10 $ & $ 2.67 \pm 0.08 $ & $ 2.01 \pm 0.09 $ & $ 0.97 \pm 0.10 $ & $  2.90 \pm 0.09   $ & $  2.73 \pm 0.10   $ & $ 2.13 \pm 0.10 $ \\
&2-Cats-G & $  3.07 \pm 0.14   $ & $  2.90 \pm 0.15   $ & $ 1.77 \pm 0.10 $ & $  3.04 \pm 0.21   $ & $ 1.55 \pm 0.11 $ & $ 0.96 \pm 0.10 $ & $  3.13 \pm 0.15   $ & $  2.60 \pm 0.12   $ & $ 2.03 \pm 0.13 $ \\
\bottomrule
\end{tabular}
\end{adjustbox}
\label{tab:synth}
\end{table*}

In both our synthetic and real data experiments, we considered the following approaches as {\em Deep Learning} baselines: Deep Archimedian Copulas (ACNET) \cite{Ling2020}; Generative Archimedian Copulas (GEN-AC) \cite{Ng2021}; Neural Likelihoods (NL) \cite{Chilinski2020}; and Implicit Generative Copulas (IGC) \cite{Janke2021}. 
Here, we use the same hyperparameters employed by the authors. 
We also considered several {\em Non-Deep Learning} baselines. These were the Parametric and Non-Parametric Copula Models from~\cite{Nagler2017}. They are referenced as: The best Parametric approach from several known parametric copulas (Par), non-penalized Bernstein estimator (Bern), penalized Bernstein estimator (PBern), penalized linear B-spline estimator (PSPL1), penalized quadratic B-spline estimator (PSPL2), transformation local likelihood kernel estimator of degree $q = 0$ (TTL0),  degree $q = 1$ (TTL1), and $q = 2$ (TTL2), and $q=2$ with nearest neughbors (TTL2nn). We also considered the recent non-parametric approach by Mukhopadhyay and Parzen~\cite{Mukhopadhyay2020} (MR), as well as the Probit transformation-based estimator from~\cite{Geenens2017}. 

To assess the performance of our model on synthetic bivariate datasets, we generated Gaussian copulas with correlations $\rho = 0.1, 0.5, 0.9$, and Clayton and Frank copulas with parameters $\theta = 1, 5, 10$. The marginal distributions were chosen as uncorrelated Normal distributions with parameters ($\mu=0, \sigma=1$). We generated 1,500 training samples and 500 test samples. Results are for test samples. Before continuing, we point out that some deep methods require datasets to be scaled to [0, 1] intervals~\cite{Ling2020,Ng2021}. Indeed, the author-supplied source code does not execute with data outside of this range (this choice was first used~\cite[Section 4.2]{Ling2020} and followed by ~\cite[Section 5.1.1]{Ng2021}). For fair comparisons, in Appendix~\ref{a:norm}, we present results for models with a min-max scaling of the data. 

The results are presented in Table~\ref{tab:synth}, which contains the average negative log-likelihood on the test set for each copula. Results were estimated with a fixed seed of \texttt{30091985}. The table also shows 95\% Bootstrap Confidence Intervals~\cite{Wasserman2004}. Following suit with several guidelines, we choose to report CIs for better interpretability on the expected range of~\cite{Gardner1986-CI,Sterne2001-CI,Altman2013-CI,Ho2019-CI,Dragicevic2016-CI,Greenland2016-CI,Imbens2021-CI}. 
\newcommand{\xmark}{\ding{55}}%

Initially, note that statistical ties (same average or overlapping intervals) are present throughout the table. The exception is for the NL~\cite{Chilinski2020} method, which severely underperforms in this setting, and the methods that did not execute (due to the required data scaling). When reading this table, it is imperative to consider that the simulated data comes from the models evaluated in the Par (first) row.
Thus, parametric methods are naturally expected to overperform. {\em Nevertheless, the statistical ties with such methods show how 2-Cats can capture this behavior even if it is a family-free approach.} 


\begin{table}[t!]
\centering
\caption{Negative log-likelihood on real datasets. $95\%$ confidence intervals (bootstrap) show. \colorbox{SpringGreen}{Green} indicates that 2-Cats was, on average, better than the best baseline (indicated in \colorbox{OrangeRed}{Red}). Underscores highlight statistical ties with the best model.}
\tiny
\begin{tabular}{llrrrr}
\toprule
& & {Boston} & {INTC-MSFT} & {GOOG-FB} \\
\midrule
\multirow{11}{*}{\rotatebox[origin=c]{90}{{\em Non-Deep Learn}}} & Par       & \underline{$-0.16 \pm 0.09$} & $-0.05 \pm 0.07$ & $-0.88 \pm 0.08$ \\
&Bern      & $-0.09 \pm 0.06$ & $-0.03 \pm 0.05$ & $-0.70 \pm 0.04$ \\
&PBern     & \underline{$-0.11 \pm 0.07$} & $-0.02 \pm 0.06$ & $-0.62 \pm 0.04$ \\
&PSPL1     & $-0.06 \pm 0.14$ & $-0.02 \pm 0.06$ & $-0.72 \pm 0.14$ \\
&PSPL2     & $-0.05 \pm 0.14$ & $-0.02 \pm 0.06$ & $-0.83 \pm 0.07$ \\
&TTL0         & \underline{$-0.16 \pm 0.08$} & $-0.06 \pm 0.06$ & $-1.00 \pm 0.07$ \\
&TTL1      & \underline{$-0.18 \pm 0.09$} & $-0.06 \pm 0.06$ & $-1.00 \pm 0.08$ \\
&TTL2      & \underline{$-0.18 \pm 0.09$} & $-0.06 \pm 0.06$ & $-0.71 \pm 0.16$ \\
&TLL2    & \underline{$-0.15 \pm 0.09$} & $-0.04 \pm 0.05$ & $-0.95 \pm 0.13$ \\
&MR        & $-0.07 \pm 0.06$ & $-0.01 \pm 0.05$ & $-0.84 \pm 0.07$ \\
&Probit      & $-0.10 \pm 0.06$ & $-0.03 \pm 0.07$ & $-0.92 \pm 0.06$ \\
\midrule
\multirow{4}{*}{\rotatebox[origin=c]{90}{{\em Deep Learn}}} & ACNet          & \underline{$-0.28 \pm 0.11$} & \cellcolor{OrangeRed}\underline{$-0.18 \pm 0.08$} & $-0.91 \pm 0.12$ \\
&GEN-AC         & \cellcolor{OrangeRed}\underline{$-0.29 \pm 0.11$} & \underline{$-0.17 \pm 0.07$} & $-0.75 \pm 0.10$ \\
&NL             & \underline{$-0.28 \pm 0.16$} & \underline{$-0.16 \pm 0.08$} & \cellcolor{OrangeRed}$-1.09 \pm 0.13$ \\
&IGC            & $0.07  \pm 0.08$ & $0.14  \pm 0.04$ & $-0.30 \pm 0.02$ \\
\midrule
\multirow{2}{*}{\rotatebox[origin=c]{90}{{\em Our}}} & 2-Cats-L      & \cellcolor{SpringGreen}\underline{$-0.30 \pm 0.09$} & \cellcolor{SpringGreen}\underline{$-0.21 \pm 0.05$} & \cellcolor{SpringGreen}\underline{$-1.75 \pm 0.04$} \\
& 2-Cats-G      & \cellcolor{SpringGreen}\underline{$-0.27 \pm 0.07$} & $-0.07 \pm 0.06$ & \cellcolor{SpringGreen}\underline{$-1.65  \pm 0.08$} \\
\bottomrule
\end{tabular}
\label{tab:real}
\end{table}

We now turn to real datasets. As was done in previous work, we employ the Boston housing dataset, the Intel-Microsoft (INTC-MSFT), and Google-Facebook (GOOG-FB) stock ratios. These pairs are commonly used in Copula research; in particular, they are the same datasets used by~\cite{Ling2020,Ng2021}. We employed the same train/test and pre-processing as previous work for fair comparisons. Thus, the data is scaled with the same code as the literature, and all deep methods are executed. 

The results are presented in Table~\ref{tab:real}. The winning 2-Cats models are highlighted in the table in \colorbox{SpringGreen}{Green}. The best, on average, baseline methods are in \colorbox{OrangeRed}{Red}. We consider a method better or worse than another when their confidence intervals do not overlap; thus, ties with the best model are shown in underscores. Overall, the table highlights the superior performance of the 2-Cats models across all datasets. Only in one setting, 2-Cats-G on INTC-MSOFT, did the model underperform.

\subsection{On the Lagrangian Term for P3}

So far, 2-Cats presents itself as a promising approach for Copula modeling. Nevertheless, to sample from Copulas, we require that the marginals of the model come from a uniform distribution. This is our primary for the Lagrangian terms used to meet P3 (see Section~\ref{sec:propsandproofs}).

We trained 2-Cats models on the three real-world datasets with and without the Lagrangian optimization of Eq~\eqref{eq:lag}. Due to space constraints, we present results for 2-Cats-L only. Here, models were trained for 1000 iterations. No early stopping was performed as our focus was on meeting P3. 

We compare both the average absolute deviations (i.e., $abs_u = \frac{1}{n}\sum_i|H_{\bm{\theta}}(u_i, 1) - u_i|$ and $abs_v = \frac{1}{n}\sum_i|H_{\bm{\theta}}(1, v_i) - v_i|$) as well as the average relative deviations (i.e., $rel_u = 100\,\frac{1}{n}\sum_i|H_{\bm{\theta}}(u_i, 1) - u_i|/u_i$ and $rel_v = 100\,\frac{1}{n}\sum_i|H_{\bm{\theta}}(1, v_i) - v_i|/v_i$) for models trained with and without constraints.

For the GOOG-FB dataset, significant gains were achieved with the constraints, i.e.: $abs_u$ went from $0.17$ to $0.09$ and $abs_v$ improved from $0.16$ to $0.03$. $rel_u$ improved from $57\%$ to $35\%$ and $rel_v$ went from $54\%$ to $14\%$. The negative log-likelihood reduced slightly when Lagrangian was used (from -1.70 to -1.14). Nevertheless, the model is still better than the baselines (see Table~\ref{tab:real}). Results for the negative log-likelihood do not exactly match the previous ones, as no early stopping was done here. Even though relative errors may appear large, note that such scores are severely sensitive to the tail of the distributions. Here, even a small deviation (e.g., 0.001 to 0.0015) incurs a large relative increase.

For INTC-MSOFT, results were similar across $u$ and $v$ dimensions. As in both versions, no significant gains were observed in $abs_u$ nor $abs_v$; these scores are pretty small (below $0.009$) in both cases. In relative terms, an increase was observed in $abs_v$ from roughly 4\% to 2\%, with $abs_u$ being 4\% in both cases. No significant changes were observed in the Boston dataset where errors were already minor $abs_u$ and $abs_v$ below $0.02$ in both instances, with relative errors ranging from 2\% to 4\% regardless of the constraints being employed. This shows that 2-Cats may already approximate the Uniform marginals even without constraints.

These results show the efficacy of our Lagrangian penalty. In some settings, such as GOOG-GB, constraints may present significant improvements. Using the Lagrangian term will depend on whether simulation is essential to the end-user (see Appendix~\ref{appn:samp}). 

\section{Conclusions} \label{s:conc}

In this paper, we presented Copula Approximating Transform models (2-Cats). 
Different from the literature when training our models; we focus not only on capturing the pseudo-likelihood of Copulas but also on meeting or approximating several other properties, such as the partial derivatives of the $C$ function. Moreover, a second major innovation is proposing Sobolev training for Copulas. 
Overall, our results show the superiority of 2-Cats on real datasets. 

A natural follow-up for 2-Cats is on using the Pair Copula Construction (PCC)~\cite{Aas2009,Czado2010} to port Vines~\cite{Czado2022,Low2018,Genest1995,Genest2007} to incorporate our method. PCC is a well-established algorithm to go from a 2D Copula to an ND One using Vines. 
We also believe that evaluating other non-negative NNs for $m_{\bm{\theta}}$ is a promising direction~\cite{Marteau2020,Rudi2021,Tsuchida2023,Sladek2023,Loconte2023,Kortvelesy2023}.


\bibliography{main}

\begin{thebibliography}{10}

\bibitem{Aas2009}
Kjersti Aas, Claudia Czado, Arnoldo Frigessi, and Henrik Bakken.
\newblock Pair-copula constructions of multiple dependence.
\newblock {\em Insurance: Mathematics and economics}, 44(2), 2009.

\bibitem{Abraj2022}
Mohomed Abraj, You-Gan Wang, and M~Helen Thompson.
\newblock A new mixture copula model for spatially correlated multiple
  variables with an environmental application.
\newblock {\em Scientific Reports}, 12(1), 2022.

\bibitem{Altman2013-CI}
Douglas Altman, David Machin, Trevor Bryant, and Martin Gardner.
\newblock {\em Statistics with confidence: confidence intervals and statistical
  guidelines}.
\newblock John Wiley \& Sons, 2013.

\bibitem{Arnold1992}
Barry~C Arnold.
\newblock {\em Multivariate logistic distributions}.
\newblock Marcel Dekker New York, 1992.

\bibitem{Bakam2023}
Yves I~Ngounou Bakam and Denys Pommeret.
\newblock Nonparametric estimation of copulas and copula densities by
  orthogonal projections.
\newblock {\em Econometrics and Statistics}, 2023.

\bibitem{Baldi2021}
Pierre Baldi.
\newblock {\em Deep learning in science}.
\newblock Cambridge University Press, 2021.

\bibitem{Bello2022}
Kevin Bello, Bryon Aragam, and Pradeep Ravikumar.
\newblock Dagma: Learning dags via m-matrices and a log-determinant acyclicity
  characterization.
\newblock In {\em NeurIPS}, 2022.

\bibitem{Bertsekas2014}
Dimitri~P Bertsekas.
\newblock {\em Constrained optimization and Lagrange multiplier methods}.
\newblock Academic press, 2014.

\bibitem{Bishop1994}
Christopher~M Bishop.
\newblock Mixture density networks.
\newblock {\em Online Report.}, 1994.

\bibitem{Cai2023}
Yongqiang Cai.
\newblock Achieve the minimum width of neural networks for universal
  approximation.
\newblock In {\em ICLR}, 2023.

\bibitem{Casella2001}
George Casella and Roger~L Berger.
\newblock {\em Statistical inference}.
\newblock Cengage Learning, 2001.

\bibitem{Cherubini2004}
Umberto Cherubini, Elisa Luciano, and Walter Vecchiato.
\newblock {\em Copula methods in finance}.
\newblock John Wiley \& Sons, 2004.

\bibitem{Chilinski2020}
Pawel Chilinski and Ricardo Silva.
\newblock Neural likelihoods via cumulative distribution functions.
\newblock In {\em UAI}, 2020.

\bibitem{Cybenko1989}
George Cybenko.
\newblock Approximation by superpositions of a sigmoidal function.
\newblock {\em Mathematics of control, signals, and systems}, 2(4), 1989.

\bibitem{Czado2010}
Claudia Czado.
\newblock Pair-copula constructions of multivariate copulas.
\newblock In {\em Copula Theory and Its Applications}. Springer, 2010.

\bibitem{Czado2022}
Claudia Czado and Thomas Nagler.
\newblock Vine copula based modeling.
\newblock {\em Annual Review of Statistics and Its Application}, 9, 2022.

\bibitem{Czarnecki2017}
Wojciech~M Czarnecki, Simon Osindero, Max Jaderberg, Grzegorz Swirszcz, and
  Razvan Pascanu.
\newblock Sobolev training for neural networks.
\newblock In {\em NeurIPS}, 2017.

\bibitem{Daniels2010}
Hennie Daniels and Marina Velikova.
\newblock Monotone and partially monotone neural networks.
\newblock {\em IEEE Transactions on Neural Networks}, 21(6), 2010.

\bibitem{Dragicevic2016-CI}
Pierre Dragicevic.
\newblock Fair statistical communication in hci.
\newblock {\em Modern statistical methods for HCI}, 2016.

\bibitem{Gardner1986-CI}
Martin Gardner and Douglas Altman.
\newblock Confidence intervals rather than p values: estimation rather than
  hypothesis testing.
\newblock {\em Br Med J (Clin Res Ed)}, 292(6522), 1986.

\bibitem{Geenens2017}
Gery Geenens, Arthur Charpentier, and Davy Paindaveine.
\newblock Probit transformation for nonparametric kernel estimation of the
  copula density.
\newblock {\em Bernoulli}, 2017.

\bibitem{Genest2007}
Christian Genest and Anne-Catherine Favre.
\newblock Everything you always wanted to know about copula modeling but were
  afraid to ask.
\newblock {\em Journal of hydrologic engineering}, 12(4), 2007.

\bibitem{Genest1995}
Christian Genest, Kilani Ghoudi, and L-P Rivest.
\newblock A semiparametric estimation procedure of dependence parameters in
  multivariate families of distributions.
\newblock {\em Biometrika}, 82(3), 1995.

\bibitem{Greenland2016-CI}
Sander Greenland, Stephen~J Senn, Kenneth~J Rothman, John~B Carlin, Charles
  Poole, Steven~N Goodman, and Douglas~G Altman.
\newblock Statistical tests, p values, confidence intervals, and power: a guide
  to misinterpretations.
\newblock {\em European journal of epidemiology}, 31, 2016.

\bibitem{Grosser2022}
Joshua Gr{\"o}{\ss}er and Ostap Okhrin.
\newblock Copulae: An overview and recent developments.
\newblock {\em Wiley Interdisciplinary Reviews: Computational Statistics},
  14(3), 2022.

\bibitem{Hall1999}
Peter Hall, Rodney~CL Wolff, and Qiwei Yao.
\newblock Methods for estimating a conditional distribution function.
\newblock {\em Journal of the American Statistical association}, 94(445), 1999.

\bibitem{Hirt2019}
Marcel Hirt, Petros Dellaportas, and Alain Durmus.
\newblock Copula-like variational inference.
\newblock In {\em NeurIPS}, 2019.

\bibitem{Ho2019-CI}
Joses Ho, Tayfun Tumkaya, Sameer Aryal, Hyungwon Choi, and Adam Claridge-Chang.
\newblock Moving beyond p values: data analysis with estimation graphics.
\newblock {\em Nature methods}, 16(7), 2019.

\bibitem{Hornik1989}
Kurt Hornik, Maxwell Stinchcombe, and Halbert White.
\newblock Multilayer feedforward networks are universal approximators.
\newblock {\em Neural networks}, 2(5), 1989.

\bibitem{Imbens2021-CI}
Guido~W Imbens.
\newblock Statistical significance, p-values, and the reporting of uncertainty.
\newblock {\em Journal of Economic Perspectives}, 35(3), 2021.

\bibitem{Janke2021}
Tim Janke, Mohamed Ghanmi, and Florian Steinke.
\newblock Implicit generative copulas.
\newblock In {\em NeurIPS}, 2021.

\bibitem{Karniadakis2021}
George~Em Karniadakis, Ioannis~G Kevrekidis, Lu~Lu, Paris Perdikaris, Sifan
  Wang, and Liu Yang.
\newblock Physics-informed machine learning.
\newblock {\em Nature Reviews Physics}, 3(6), 2021.

\bibitem{Kingma2014}
Diederik~P Kingma and Jimmy Ba.
\newblock Adam: A method for stochastic optimization.
\newblock In {\em ICLR}, 2015.

\bibitem{Klambauer2017}
G{\"u}nter Klambauer, Thomas Unterthiner, Andreas Mayr, and Sepp Hochreiter.
\newblock Self-normalizing neural networks.
\newblock In {\em NeurIPS}, 2017.

\bibitem{Kobyzev2020}
Ivan Kobyzev, Simon~JD Prince, and Marcus~A Brubaker.
\newblock Normalizing flows: An introduction and review of current methods.
\newblock {\em IEEE transactions on pattern analysis and machine intelligence},
  43(11), 2020.

\bibitem{Kortvelesy2023}
Ryan Kortvelesy.
\newblock Fixed integral neural networks.
\newblock {\em arXiv preprint arXiv:2307.14439}, 2023.

\bibitem{Lindgren2013}
Georg Lindgren, Holger Rootz{\'e}n, and Maria Sandsten.
\newblock {\em Stationary stochastic processes for scientists and engineers}.
\newblock CRC press, 2013.

\bibitem{Ling2020}
Chun~Kai Ling, Fei Fang, and J~Zico Kolter.
\newblock Deep archimedean copulas.
\newblock In {\em NeurIPS}, 2020.

\bibitem{Liu2023}
Yucong Liu.
\newblock Neural networks are integrable.
\newblock {\em arXiv preprint arXiv:2310.14394}, 2023.

\bibitem{Loconte2023}
Lorenzo Loconte, Stefan Mengel, Nicolas Gillis, and Antonio Vergari.
\newblock Negative mixture models via squaring: Representation and learning.
\newblock In {\em The 6th Workshop on Tractable Probabilistic Modeling}, 2023.

\bibitem{Low2018}
Rand Kwong~Yew Low, Jamie Alcock, Robert Faff, and Timothy Brailsford.
\newblock Canonical vine copulas in the context of modern portfolio management:
  Are they worth it?
\newblock {\em Asymmetric Dependence in Finance: Diversification, Correlation
  and Portfolio Management in Market Downturns}, 2018.

\bibitem{Lu2021}
Lu~Lu, Raphael Pestourie, Wenjie Yao, Zhicheng Wang, Francesc Verdugo, and
  Steven~G Johnson.
\newblock Physics-informed neural networks with hard constraints for inverse
  design.
\newblock {\em SIAM Journal on Scientific Computing}, 43(6), 2021.

\bibitem{Lu2020}
Yulong Lu and Jianfeng Lu.
\newblock A universal approximation theorem of deep neural networks for
  expressing probability distributions.
\newblock In {\em NeurIPS}, 2020.

\bibitem{Lu2017}
Zhou Lu, Hongming Pu, Feicheng Wang, Zhiqiang Hu, and Liwei Wang.
\newblock The expressive power of neural networks: A view from the width.
\newblock In {\em NeurIPS}, 2017.

\bibitem{Marteau2020}
Ulysse Marteau-Ferey, Francis Bach, and Alessandro Rudi.
\newblock Non-parametric models for non-negative functions.
\newblock In {\em NeurIPS}, 2020.

\bibitem{Modiri2020}
Sadegh Modiri, Santiago Belda, Mostafa Hoseini, Robert Heinkelmann, Jos{\'e}~M
  Ferr{\'a}ndiz, and Harald Schuh.
\newblock A new hybrid method to improve the ultra-short-term prediction of
  lod.
\newblock {\em Journal of geodesy}, 94, 2020.

\bibitem{Mukhopadhyay2020}
Subhadeep Mukhopadhyay and Emanuel Parzen.
\newblock Nonparametric universal copula modeling.
\newblock {\em Applied Stochastic Models in Business and Industry}, 36(1),
  2020.

\bibitem{Naaman2021}
Michael Naaman.
\newblock On the tight constant in the multivariate
  dvoretzky--kiefer--wolfowitz inequality.
\newblock {\em Statistics \& Probability Letters}, 173, 2021.

\bibitem{Nagler2017}
Thomas Nagler, Christian Schellhase, and Claudia Czado.
\newblock Nonparametric estimation of simplified vine copula models: comparison
  of methods.
\newblock {\em Dependence Modeling}, 5(1), 2017.

\bibitem{Nelsen2006}
Roger~B Nelsen.
\newblock {\em An introduction to copulas}.
\newblock Springer, 2006.

\bibitem{Ng2021}
Yuting Ng, Ali Hasan, Khalil Elkhalil, and Vahid Tarokh.
\newblock Generative archimedean copulas.
\newblock In {\em UAI}, 2021.

\bibitem{Nguyen2020}
T~Tin Nguyen, Hien~D Nguyen, Faicel Chamroukhi, and Geoffrey~J McLachlan.
\newblock Approximation by finite mixtures of continuous density functions that
  vanish at infinity.
\newblock {\em Cogent Mathematics \& Statistics}, 7(1), 2020.

\bibitem{Papamakarios2017}
George Papamakarios, Theo Pavlakou, and Iain Murray.
\newblock Masked autoregressive flow for density estimation.
\newblock In {\em NeurIPS}, 2017.

\bibitem{Platt1987}
John Platt and Alan Barr.
\newblock Constrained differential optimization.
\newblock In {\em NeurIPS}, 1987.

\bibitem{Rudi2021}
Alessandro Rudi and Carlo Ciliberto.
\newblock Psd representations for effective probability models.
\newblock In {\em NeurIPS}, 2021.

\bibitem{Salvadori2004}
G~Salvadori and Carlo De~Michele.
\newblock Frequency analysis via copulas: Theoretical aspects and applications
  to hydrological events.
\newblock {\em Water resources research}, 40(12), 2004.

\bibitem{Schepsmeier2014}
Ulf Schepsmeier and Jakob St{\"o}ber.
\newblock Derivatives and fisher information of bivariate copulas.
\newblock {\em Statistical papers}, 55(2), 2014.

\bibitem{Sill1997}
Joseph Sill.
\newblock Monotonic networks.
\newblock In {\em NeurIPS}, 1997.

\bibitem{Silverman1986}
Bernard~W Silverman.
\newblock {\em Density estimation for statistics and data analysis}, volume~26.
\newblock CRC press, 1986.

\bibitem{Sklar1996}
Abe Sklar.
\newblock Random variables, distribution functions, and copulas: a personal
  look backward and forward.
\newblock {\em Lecture notes-monograph series}, 1996.

\bibitem{Sklar1959}
M~Sklar.
\newblock Fonctions de r{\'e}partition {\`a} n dimensions et leurs marges.
\newblock {\em Annales de l'ISUP}, 8(3), 1959.

\bibitem{Sladek2023}
Aleksanteri~Mikulus Sladek, Martin Trapp, and Arno Solin.
\newblock Encoding negative dependencies in probabilistic circuits.
\newblock In {\em The 6th Workshop on Tractable Probabilistic Modeling}, 2023.

\bibitem{Sohn2015}
Kihyuk Sohn, Honglak Lee, and Xinchen Yan.
\newblock Learning structured output representation using deep conditional
  generative models.
\newblock In {\em NeurIPS}, 2015.

\bibitem{Sohrabian2021}
B~Sohrabian.
\newblock Geostatistical prediction through convex combination of archimedean
  copulas.
\newblock {\em Spatial Statistics}, 41, 2021.

\bibitem{Sterne2001-CI}
Jonathan~AC Sterne and George~Davey Smith.
\newblock Sifting the evidence—what's wrong with significance tests?
\newblock {\em Physical therapy}, 81(8), 2001.

\bibitem{Tagasovska2023}
Natasa Tagasovska, Firat Ozdemir, and Axel Brando.
\newblock Retrospective uncertainties for deep models using vine copulas.
\newblock In {\em AISTATS}, 2023.

\bibitem{Tsay2023}
Wen-Jen Tsay and Peng-Hsuan Ke.
\newblock A simple approximation for the bivariate normal integral.
\newblock {\em Communications in Statistics-Simulation and Computation}, 52(4),
  2023.

\bibitem{Tsuchida2023}
Russell Tsuchida, Cheng~Soon Ong, and Dino Sejdinovic.
\newblock Squared neural families: A new class of tractable density models.
\newblock {\em arXiv preprint arXiv:2305.13552}, 2023.

\bibitem{Walsh1975}
GR~Walsh.
\newblock Saddle-point property of lagrangian function.
\newblock {\em Methods of Optimization. New York: John Wiley \& Sons}, 1975.

\bibitem{Wasserman2004}
Larry Wasserman.
\newblock {\em All of statistics: a concise course in statistical inference},
  volume~26.
\newblock Springer, 2004.

\bibitem{Wehenkel2019}
Antoine Wehenkel and Gilles Louppe.
\newblock Unconstrained monotonic neural networks.
\newblock In {\em NeurIPS}, 2019.

\bibitem{Yang2019}
Liuyang Yang, Jinghuai Gao, Naihao Liu, Tao Yang, and Xiudi Jiang.
\newblock A coherence algorithm for 3-d seismic data analysis based on the
  mutual information.
\newblock {\em IEEE Geoscience and Remote Sensing Letters}, 16(6), 2019.

\bibitem{Zhang2017}
Qingyang Zhang and Xuan Shi.
\newblock A mixture copula bayesian network model for multimodal genomic data.
\newblock {\em Cancer Informatics}, 16, 2017.

\end{thebibliography}
\pagebreak
\appendix
\section{On Derivatives and Integrals of NNs} \label{appn}

In this appendix, we present an example of the issue of approximating derivatives/integrals of NNs. 

\subsection{An Example on why approximating derivatives and integrals fail}

We begin with a pictorial example of the issue. Before doing so, we take some time to revise the universal approximation properties of NNs~\cite{Cybenko1989,Hornik1989,Lu2020}. We also shift our focus from Copulas for this argument. It is expected to state that the universal approximation theorem (UAT) guarantees that functions are approximated with NNs. However, given that when learning models from training data, it would be more correct to say that function {\it evaluations at training points} are approximated.

\begin{figure}
    \centering
    \includegraphics[width=0.95\linewidth]{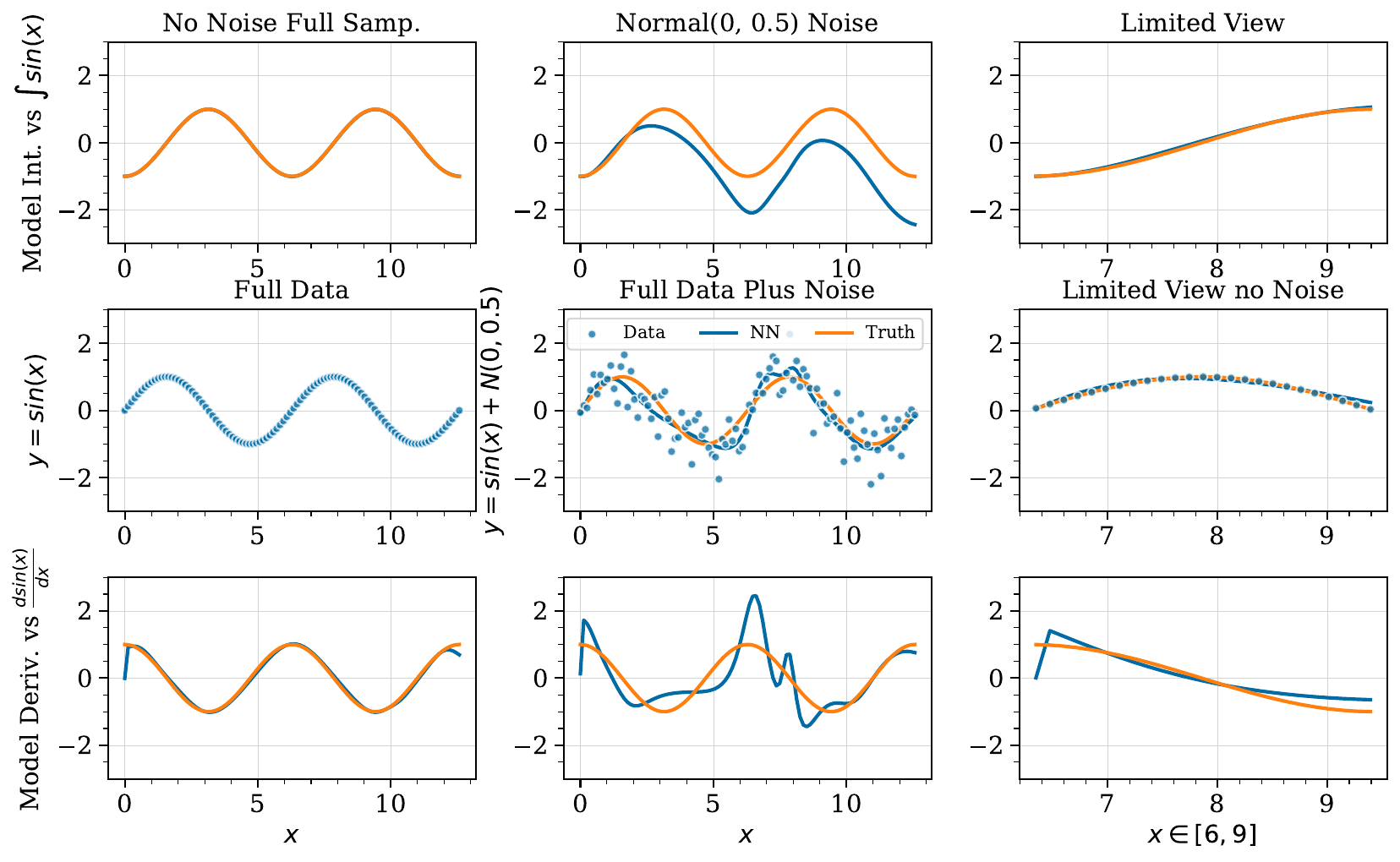}
    \caption{Models vs Derivatives. Here, we compare sample data from the simple $y = sin(x)$ process. Data is sampled in three settings, a ``full'' dataset with a larger range (left), a noisy dataset (middle), and a limited view (right). NN models (lines) approximate the data points in every case. However, the derivatives/integrals of the model do not approximate function derivatives (last row).} 
    \label{fig:deriv}
\end{figure}

Consider a simple 4-Layer Relu NN with layers of 128 neurons. This will be our model. To avoid overfitting, this model will be optimized with a validation set. The problem with relying only on the UAT is shown in Figure~\ref{fig:deriv}. The figure shows a ground truth data-generating process of $y = \sin(x) + \eta$, where $\eta$ is some noise term. On the first column, when $\eta = 0$, the model will approximate the data generating function (middle row), its derivative (bottom row), and integral (top row). When $\eta \sim N(0, 0.5)$, the model {\em will} approximate the data points, but {\em not necessarily} the derivatives. 

This example may seem to contradict~\cite{Czarnecki2017,Hornik1989} others that looked into the universal approximation properties of NN regarding derivative. We point out that this is {\em not} the case. In Appendix~\ref{appn:deriv}, we show how the variance will increase when estimating the derivative, leading to the figure's oscillating behavior. From the figure, also note that integrals are approximated up to an asymptotic constant (see Appendix~\ref{a:int}, and ~\cite{Liu2023}); however, we still need to control for this constant (note the growing deviation on the top plot of the middle row).

We now present some results on the Integral (exemplified on the top row of Figure~\ref{fig:deriv}) and Derivative (bottom row) that tie in with our example. Furthermore, we briefly discuss the last column of the figure. Notice that on the extreme left and extreme right, the derivative of the network begins to diverge from the underlying process. The NN was not trained with points sampled after these boundaries. Under data shifts (e.g., data from a similar process not seen on the training set), the issues we point out throughout this Appendix should increase.

\subsection{On the Integral of Neural Networks} \label{a:int}

To present a view of the theoretical bounds on the error of approximation of integrals, let $\hat{f}_{\bm{\theta}}(x)$ be some NN parametrized by $\bm{\theta}$. Also, let $f(x)$ represent our training dataset. We also make the reasonable assumption our dataset comes from some real function, $f^{\ast}(x)$, sampled under a symmetric additive noise, i.e., $f(x) = f^{\ast}(x) + \eta$, with an expected value equal to zero: $\mathbb{E}[\eta] = 0$. Gaussian noise is such a case.

If a NN is a universal approximator~\cite{Cybenko1989,Hornik1989,Lu2021,Cai2023,Lu2017} and it is trained on the dataset $f(x)$, we reach:
$$\lVert \hat{f}_{\bm{\theta}}(x) - f(x) \rVert_p < \epsilon_1$$
\noindent Most UAT proofs are valid for any $p$ norm ($p$ norms bound one another), and assuming the 1 norm, we reach: $\hat{f}_{\bm{\theta}}(x) = f(x) \pm \epsilon$ for some constant $\epsilon$. As a consequence, we reach:
$$\hat{f}_{\bm{\theta}}(x) = f^{\ast}(x) + \eta \pm \epsilon_1$$

Given that integration is a linear operator, we can show that:
$$E_{\eta}[\int \hat{f}_{\bm{\theta}}(x)] = \int f^{\ast}(x) \pm \epsilon_2.$$
To do so, we shall use the law of the unconscious statistician as follows:

\begin{align*}
g(x) &= \int \hat{f}{_{\theta}}(x) \,dx\\
E_{\eta}[g(x)] &= E_{\eta}[\int \hat{f}_\theta(x) dx] \\
&= \int \int \hat{f}_\theta(x) p_{\eta}(n) \, dx\, dn \\ 
&= \int \int p_{\eta}(n)\hat{f}_\theta(x) \, dn\, dx \\
&= \int \int p_{\eta}(n)\big( f^{\ast}(x) + \eta \pm \epsilon_1 \big) \, dn\, dx \\
&= \int \int p_{\eta}(n) f^{\ast}(x) + p_{\eta}(n) \eta \pm p_{\eta}(n) \epsilon_1 \, dn\, dx \\
&= \int E_\eta[f^{\ast}(x)] + E_\eta[\eta] \pm E_\eta[\epsilon_1] dx \\&= \int E_\eta[f^{\ast}(x)] \pm E_\eta[\epsilon_1] dx \\&= \int f^{\ast}(x) \pm \epsilon_1 dx \\&= \int f^{\ast}(x) dx \pm \int \epsilon_1 dx = \int f^{\ast}(x) dx \pm \epsilon_2 dx
\end{align*}
\noindent where $\epsilon_2$ is a constant as $\epsilon_1$ is also a constant. However, this constant will grow depending on the integrating interval. This growing error is visible on the plot on the center top of Figure~\ref{fig:deriv}.
Controlling for $\epsilon_2$ will depend on the dataset's quality and the training procedure. 

A similar result is discussed in~\cite{Liu2023}.

\subsection{On the Derivative of Underlying Process and its impact on  Neural Networks} \label{appn:deriv}

To understand the oscillations on the derivatives of the example, let us rewrite the underlying process that generated the data as a Gaussian process. That is: $Y(x) = sin(x + \phi) + \eta$. $Y$ is the stochastic process and $\phi \sim Uniform(-2\pi, 2\pi)$ is some phase component necessary for stationarity. The variance of this process will be given by the kernel function $K(\tau)$, which is now controlled by parameter $\tau$. Given that this process has the same variance over time, $x$, it is stationary, and a stationary Gaussian process has $K(\tau) = e^{-\frac{\tau}{2}}$~\cite{Lindgren2013}. Before continuing, we note that in our example, we sampled from a Gaussian with a standard deviation of $\sigma=0.5$ and variance of $\sigma^2=0.25$. This leads to a value of $\tau$ that is of $\tau = \approx 2.8$ since $e^{-\frac{2.8}{2}} \approx 0.25$.

Now, let $Y'(x)$ be the derivative of this process. This may be defined as:
\begin{align*}
\frac{Y(x+h) - Y(x)}{h} &\rightarrow Y'(x)
\end{align*}
if:
\begin{align*}
E\big[(\frac{Y(x+h) - Y(x)}{h} - Y'(x))^2 \big] &\rightarrow 0
\end{align*}
when: $h \rightarrow 0$.

Given that our process is stationary and has the expected value of the noise term as zero, we can estimate the value of the derivative process as zero, i.e., $E[Y'(x)] = 0$~\cite{Lindgren2013}. We also have that the variance is equal to the second derivative of $K(\tau)$: $Var[Y'(x)] = \frac{\partial K}{\partial \tau \partial \tau}$~\cite{Lindgren2013}.

From here, we can derive $K$ twice to show that: $Var[Y'(x)] = -e^{-\frac{\tau^2}{2}} + \tau^2 e^{-\frac{\tau^2}{2}}$. Recall that $Var[Y(x)] = \sigma^2 = 0.25$ and $\tau \approx 2.8$. Now, by solving a simple inequality, we can show that the variance of the derivative process, $Var[Y'(x)]$, is greater than the variance of the original process $Var[Y(x)]$, i.e., $Var[Y'(x)] > Var[Y(x)]$ when $\tau > \sqrt{2} \approx 1.4$ (as is our example where $\tau \approx 2.8$).

This result explains the oscillation in the bottom middle plot of Figure~\ref{fig:deriv}. The derivative process has a higher variation; thus, we expect a higher variance in estimations of the derivative.

Overall, even if an NN is a universal approximation of derivatives~\cite{Czarnecki2017,Hornik1989}, the variance of the noise term will likely increase (here we began with a variance of 0.25 and it is already sufficient to show such an increase)  even for the derivative of the NN.

\section{Sampling from 2-Cats} \label{appn:samp}

We now detail how to sample from 2-Cats. Let the first derivatives of 2-Cats be: $h_v(u) = \Pr[U \leq u \mid V = v] = \frac{\partial H_{\bm{\theta}}(u, v)}{\partial v}$, and $h_u(v) = \Pr[V \leq v \mid U = u] = \frac{\partial H_{\bm{\theta}}(u, v)}{\partial u}$. Now, let $h'_v(u) = \frac{\partial^2 H_{\bm{\theta}}(u, v)}{\partial v^2}$ be the first derivative of $h_v(u)$, and similarly $h'_u(v) = \frac{\partial^2 H_{\bm{\theta}}(u, v)}{\partial u^2}$. These derivatives are readily available in the Hessian matrix that we estimate symbolically for 2-Cats (see Section~\ref{sec:propsandproofs}).

Now, let us determine the inverse of $h_u(v)$, that is: $h^{-1}_u(v_{i})$ (the subscript stands for inverse). The same arguments are valid for $h_v(u)$ and $h^{-1}_v(u_{i})$, and thus we omit them. Notice that with this inverse, we can sample from the CDF defined by $h_u(v)$ using the well-known Inverse transform sampling. Now, notice that by definition, $h_u(v)$ is already the derivative of $H_{\bm{\theta}}(u, v)$ with regards to $u$.

Using a Legendre transform\footnote{\url{https://en.wikipedia.org/wiki/Inverse_function_rule}}, the inverse of a derivative is:
$$ h^{-1}_u(v_{i}) = \big(  h'_u(v_{i}) \big)^{-1} = \big(  \frac{\partial^2 H_{\bm{\theta}}(u, v_{i})}{\partial u^2} \big)^{-1} $$
\noindent Where, again, the derivative inside the parenthesis is readily available to us via symbolic computation.

With such results in hand, the algorithm for sampling is:
\begin{enumerate}
    \item Generate two independent values $u \sim Uniform(0, 1)$ and $v_{i} \sim Uniform(0, 1)$.
    \item Set $v = h^{-1}_u(v_{i})$. This is an Inverse transform sampling for $v$
    \item Now we have the $(u, v)$ pair. We can again use Inverse transform sampling to determine:
    \begin{enumerate}
        \item $x_1 = F^{-1}_{x_1}(u)$
        \item $x_2 = F^{-1}_{x_2}(v)$
    \end{enumerate}
\end{enumerate}

\section{Ablation Study} \label{appn:ablation}

We now provide an ablation study to understand the impact of our three losses. In this study, our model was trained without the Lagrangian terms of Property P3. 

\begin{table}[t!]
\centering
\caption{Ablation Study}
\begin{tabular}{llll}
{\bf BOSTON} \\
\toprule
 & $L^C$ & $L^{dC}$ & $L^{c}$ \\
\midrule
Only $L^c$ for training & 0.115 & 0.033 & -0.448 \\
$L^c$ and $L^{dC}$ for training & {\bf 0.107} & 0.035 & -0.454 \\
$L^c$ and $L^{C}$ for training & 0.103 & 0.023 & -0.236 \\
All Three Losses (paper) & {\bf 0.107} & {\bf 0.020} & {\bf -0.630} \\
\bottomrule
\end{tabular}
\vspace{1em}
\begin{tabular}{llll}
\,\\
{\bf INCT-MSOFT} \\
\toprule
 & $L^C$ & $L^{dC}$ & $L^{c}$ \\
\midrule
Only $L^c$ for training & 0.131 & 0.007 & -0.327 \\
$L^c$ and $L^{dC}$ for training & 0.141 & 0.004 & -0.314 \\
$L^c$ and $L^{C}$ for training & 0.137 & 0.007 & -0.320 \\
All Three Losses (paper) & 0.141 & 0.018 & {\bf -0.402} \\
\bottomrule
\end{tabular}
\vspace{1em}
\begin{tabular}{llll}
{\bf GOOG-FB} \\
\toprule
 &  $L^C$ & $L^{dC}$ & $L^{c}$ \\
\midrule
Only $L^c$ for training & 0.166 & {\em 6.88} & -3.324 \\
$L^c$ and $L^{dC}$ for training & 0.133 & 0.066 & -2.163 \\
$L^c$ and $L^{C}$ for training & 0.195 & {\em 2.50} & -3.201 \\
All Three Losses (paper) & {\bf 0.136} & 0.087 & -2.881 \\
\bottomrule
\end{tabular}
\label{t:ablation}
\end{table}

The three tables (see Table~\ref{t:ablation}) of this section below present the values for $L^C$ (the squared error of the cumulative function C), $L^{dC}$ (the squared error for the first derivatives of C), and $L^c$ (the copula likelihood). We present one table for each real-world dataset.

$L^c$ is the score of interest when comparing models and the one we report in our paper. However, as discussed in Appendix~\ref{appn} of our manuscript, when an NN approximates one aspect of a function (here being the Copula density $L^c$), the NN may miss other elements, such as the integral and derivatives of the function. This is the reason why we now present results for all three metrics. The metrics are in the columns of the tables. Each row presents a different 2-Cats approach using: (1) only $L^c$ for training (the metric of most interest); (2) $L^c$ and $L^{dC}$; $L^c$ and $L^{C}$; and, (3) all three losses (as in the main text).

From the tables, we can see that every loss plays a role when training the model. When using all three losses, gains/ties are achieved in 5 out of 9 cases ({\bf bold}). When ignoring such terms, large losses ({\em italic}) will also arise

\section{Conjecture: A Mixture of 2-Cats is a Universal Copula Approximator} \label{appn:uat}

{\em With sufficient components}, any other CDF (or density function) may be approximated by a mixture of other CDFs~\cite{Nguyen2020}. As a consequence, a mixture of 2-Cats of the form:
$$H_{\bm{m}, \bm{\Theta}}(u, v) = \sum_{i=1}^k w_i H_{i, \bm{\theta_i}}(u, v),$$
\noindent where $H_{i, \bm{\theta_i}}(u, v)$ is the $i$-th 2-Cats model. $\bm{\Theta}$ is the parameter vector comprised of concatenating the individual parameters of each mixture component: $\bm{\Theta} = [\bm{\theta_1}, \bm{\theta_2}, \dots, \bm{\theta_k}]$, and $\bm{m} = [m_1, m_2, \dots m_k]$ are real numbers related to the mixture weights.

 This model may be trained as follows:
 \begin{enumerate}
    \item Map mixture parameters to the $k-1$ Simplex: $\mathbf{w} = \text{softmax}(\bm{\theta}_w)$. $w_i$ if the $i$-th position of this vector.
    \item When training, backpropagate to learn  $\bm{\Theta}$ and $\bm{m}$. 
\end{enumerate}

This is similar to the Mixture Density NN~\cite{Bishop1994}. Over the last few years, mixture of Copulas approaches has been gaining traction in several fields~\cite{Abraj2022,Sohrabian2021,Zhang2017}; here, we are proposing an NN variation. By definition, this is a valid Copula.

\section{Scaled Synthetic Data} \label{a:norm}

\begin{table*}[t!]
\centering
\caption{Results for Synthetic Data. Copula parameters are shown on the header. Marginals come from two Normal distributions ($\mu=0, \sigma=1$). $95\%$ confidence intervals (bootstrap) show. \colorbox{SpringGreen}{Green} indicates that a 2-Cats model was, on average, better than the best baseline (indicated by \colorbox{OrangeRed}{Red}). Underscores highlight the best statistical ties with the best model. {\bf We point out that training and test data are independently normalized to the [0, 1] range. This is a necessity of some deep models such as~\cite{Ling2020,Ng2021}} that do not execute on data outside this range (this choice was first used~\cite[Section 4.2]{Ling2020} and followed by ~\cite[Section 5.1.1]{Ng2021}).}
\begin{adjustbox}{max width=\textwidth}
\begin{tabular}{llrrrrrrrrr}
\toprule
&  & \multicolumn{3}{c}{Gaussian ($\rho$)} & \multicolumn{3}{c}{Clayton ($\theta$)}  & \multicolumn{3}{c}{Frank/Joe ($\theta$)} \\
&  & 0.1 & 0.5 & 0.9 & 1 & 5 & 10 & 1 & 5 & 10 \\
\midrule
\multirow{11}{*}{\rotatebox[origin=c]{90}{{\em Non-Deep Learn}}} & Par & \underline{$-0.55 \pm 0.12$} & \underline{$-0.55 \pm 0.13$} & $-0.55 \pm 0.16$ & \cellcolor{OrangeRed}\underline{$-1.12 \pm 0.09$} & $-1.12 \pm 0.16$ & $-1.12 \pm 0.21$ & \cellcolor{OrangeRed}\underline{$-0.86 \pm 0.09$} & \cellcolor{OrangeRed}$-0.86 \pm 0.12$ & \cellcolor{OrangeRed}$-0.86 \pm 0.15$ \\
&Bern & \underline{$-0.54 \pm 0.12$} & \underline{$-0.54 \pm 0.13$} & $-0.54 \pm 0.14$ & \underline{$-1.10 \pm 0.08$} & \underline{$-1.10 \pm 0.12$} & $-1.10 \pm 0.18$ & \underline{$-0.85 \pm 0.09$} & $-0.85 \pm 0.12$ & $-0.85 \pm 0.14$ \\
&PBern & \underline{$-0.55 \pm 0.11$} & \underline{$-0.55 \pm 0.14$} & $-0.55 \pm 0.14$ & \underline{$-1.11 \pm 0.08$} & \underline{$-1.11 \pm 0.12$} & $-1.11 \pm 0.18$ & \cellcolor{OrangeRed}\underline{$-0.86 \pm 0.09$} & \cellcolor{OrangeRed}$-0.86 \pm 0.13$ & \cellcolor{OrangeRed}$-0.86 \pm 0.15$ \\
&PSPL1 & \underline{$-0.55 \pm 0.11$} & \underline{$-0.55 \pm 0.13$} & $-0.55 \pm 0.14$ & \underline{$-1.11 \pm 0.07$} & \underline{$-1.11 \pm 0.15$} & $-1.11 \pm 0.30$ & \cellcolor{OrangeRed}\underline{$-0.86 \pm 0.09$} & \cellcolor{OrangeRed}$-0.86 \pm 0.12$ & \cellcolor{OrangeRed}$-0.86 \pm 0.18$ \\
&PSPL2 & \underline{$-0.55 \pm 0.11$} & \underline{$-0.55 \pm 0.13$} & $-0.55 \pm 0.15$ & \underline{$-1.10 \pm 0.08$} & \underline{$-1.10 \pm 0.14$} & $-1.10 \pm 0.24$ & \cellcolor{OrangeRed}\underline{$-0.86 \pm 0.09$} & \cellcolor{OrangeRed}$-0.86 \pm 0.12$ & \cellcolor{OrangeRed}$-0.86 \pm 0.17$ \\
&TTL0 & \underline{$-0.54 \pm 0.12$} & \underline{$-0.54 \pm 0.14$} & $-0.54 \pm 0.15$ & \underline{$-1.11 \pm 0.09$} & \underline{$-1.11 \pm 0.16$} & $-1.11 \pm 0.27$ & \cellcolor{OrangeRed}\underline{$-0.86 \pm 0.08$} & \cellcolor{OrangeRed}$-0.86 \pm 0.13$ & \cellcolor{OrangeRed}$-0.86 \pm 0.16$ \\
&TLL1 & \underline{$-0.53 \pm 0.12$} & \underline{$-0.53 \pm 0.14$} & $-0.53 \pm 0.16$ & \cellcolor{OrangeRed}\underline{$-1.12 \pm 0.08$} & $-1.12 \pm 0.17$ & $-1.12 \pm 0.32$ & \underline{$-0.85 \pm 0.09$} & $-0.85 \pm 0.12$ & $-0.85 \pm 0.17$ \\
&TLL2 & \underline{$-0.54 \pm 0.12$} & \underline{$-0.54 \pm 0.14$} & $-0.54 \pm 0.15$ & \underline{$-1.12 \pm 0.09$} & \underline{$-1.12 \pm 0.19$} & $-1.12 \pm 0.40$ & \cellcolor{OrangeRed} \cellcolor{OrangeRed}\underline{$-0.86 \pm 0.09$}& \cellcolor{OrangeRed}$-0.86 \pm 0.13$ & \cellcolor{OrangeRed}$-0.86 \pm 0.16$ \\
&TLL2nn & \underline{$-0.55 \pm 0.12$} & \underline{$-0.55 \pm 0.14$} & $-0.55 \pm 0.15$ & \cellcolor{OrangeRed}\underline{$-1.12 \pm 0.08$} & $-1.12 \pm 0.17$ & $-1.12 \pm 0.40$ & \cellcolor{OrangeRed}\underline{$-0.86 \pm 0.09$} & \cellcolor{OrangeRed}$-0.86 \pm 0.13$ & \cellcolor{OrangeRed}$-0.86 \pm 0.15$ \\
&MR & \underline{$-0.54 \pm 0.12$} & \underline{$-0.54 \pm 0.13$} & $-0.54 \pm 0.14$ & \underline{$-1.08 \pm 0.08$} & \underline{$-1.08 \pm 0.14$} & $-1.08 \pm 0.34$ & \underline{$-0.84 \pm 0.09$} & $-0.84 \pm 0.12$ & $-0.84 \pm 0.15$ \\
&Probit & \underline{$-0.55 \pm 0.12$} & \underline{$-0.55 \pm 0.13$} & $-0.55 \pm 0.15$ & \underline{$-1.10 \pm 0.08$} & \underline{$-1.10 \pm 0.14$} & $-1.10 \pm 0.29$ & \underline{$-0.85 \pm 0.09$} & $-0.85 \pm 0.12$ & $-0.85 \pm 0.16$ \\
\midrule
\multirow{3}{*}{\rotatebox[origin=c]{90}{{\em Deep Learn}}} &ACNet          & $-0.06 \pm 0.09$ & $-0.29 \pm 0.09$ & $-1.05 \pm 0.08$ & $-0.51 \pm 0.07$ & $-0.91 \pm 0.11$ & \cellcolor{OrangeRed}$-1.35 \pm 0.04$ & $-0.13 \pm 0.09$ & $-0.49 \pm 0.08$ & $-0.58 \pm 0.10$ \\
&GEN-AC         & $-0.06 \pm 0.08$ & $-0.29 \pm 0.08$ & $-1.06 \pm 0.07$ & $-0.52 \pm 0.06$ & $-0.93 \pm 0.12$ & $-1.34 \pm 0.05$ & $-0.15 \pm 0.07$ & $-0.49 \pm 0.08$ & $-0.56 \pm 0.10$ \\
&NL             & $-0.34 \pm 0.09$ & $-0.43 \pm 0.07$ & \cellcolor{OrangeRed}\underline{$-1.07 \pm 0.07$} & $-0.64 \pm 0.07$ & \cellcolor{OrangeRed}$-1.17 \pm 0.13$ & $-1.03 \pm 0.07$ & $-0.42 \pm 0.07$ & $-0.59 \pm 0.08$ & $-0.57 \pm 0.06$ \\
&IGC            & \cellcolor{OrangeRed}\underline{$-0.58 \pm 0.07$} & \cellcolor{OrangeRed}\underline{$-0.67 \pm 0.06$} & $-0.86 \pm 0.06$ & $-0.66 \pm 0.06$ & $-0.88 \pm 0.06$ & $-0.83 \pm 0.06$ & $-0.66 \pm 0.06$ & $-0.73 \pm 0.05$ & $-0.75 \pm 0.07$ \\
\midrule
\multirow{2}{*}{\rotatebox[origin=c]{90}{{\em Our}}} &2-Cats-L      & \underline{$-0.48 \pm 0.14$} & \cellcolor{SpringGreen}\underline{$-0.63 \pm 0.13$} & \cellcolor{SpringGreen}\underline{$-1.20 \pm 0.23$} & $-0.64 \pm 0.14$ & \underline{$-1.24 \pm 0.20$} & \cellcolor{SpringGreen}$-1.69 \pm 0.19$ & $-0.52 \pm 0.17$ & \cellcolor{SpringGreen}\underline{$-1.34 \pm 0.17$} & \cellcolor{SpringGreen}\underline{$-1.34 \pm 0.14$} \\
&2-Cats-G      & \underline{$-0.38  \pm 0.15$} & \underline{$-0.45 \pm 0.19$} & \underline{$-1.04 \pm 0.28$} & $-0.44 \pm 0.19$ & \cellcolor{SpringGreen}\underline{$-1.27 \pm 0.30$} & $-1.06 \pm 0.18$ & $-0.51 \pm 0.15$ & \underline{$-1.12 \pm 0.09$} & \underline{$-1.06 \pm 0.18$} \\





\bottomrule
\end{tabular}
\label{t:zonesynth}
\end{adjustbox}
\end{table*}

Table~\ref{t:zonesynth} presents results without scaling the input data as is done in Deep Learning methods. The table presents the average negative log-likelihood and the 95\% confidence interval. 

The colors and highlights on this table match the main text. The winning 2-Cats models are highlighted in the table in \colorbox{SpringGreen}{Green}. The best, on average, baseline methods are in \colorbox{OrangeRed}{Red}. The table shows that 2-Cats is better than baseline methods when the dependency ($\rho$ or $\theta$) increases. For small dependencies, methods not based on Deep Learning outperform Deep ones (including 2-Cats). Nevertheless, 2-Cats is the winning method in 6 out of 9 datasets and is tied with the best in one case.



\section{A Flexible Variation of the Model} \label{appn:var}

A Flexible 2-Cats model works similarly to our 2-Cats. However, the transforms are different:
\begin{enumerate}
    \item Let $m_{\bm{\theta}}: \mathbb{I}^2 \mapsto \mathbb{R}+$ be an MLP outputs positive numbers. We achieve this by employing Elu plus one activation in every layer as in~\cite{Wehenkel2019}.
    \item Define the transforms: \begin{align*}t_v(u) = \int_{0}^{u} m_{\bm{\theta}}(x, v)\, dx; \,\,
t_u(v) = \int_{0}^{v} m_{\bm{\theta}}(y, u) \,dy.\end{align*}
    \item The 2-Cats-FLEX hypothesis is now defined as the function: $H_{\bm{\theta}}(u, v) = G\big((t_v(u), t_u(v)\big)$, where $G(x, y)$ is any bivariate CDF on the $\mathbb{R}^2$ domain (e.g., the Bivariate Logistic or Bivariate Gaussian).
    
\end{enumerate}


This model meets {\bf P2}, but not {\bf P1} nor {\bf P3}. As already stated in the introduction, the fact that $t_v(u)$ and $t_u(v)$ are monotonic, one-to-one, guarantees that $G(t_v(u), t_u(v))$ defines a {\em valid} probability transform to a new CDF. A major issue with this approach is that we have no guarantee that the model's derivatives are conditional cumulative functions (first derivatives) or density functions (second derivatives). That's why we call it Flexible (FLEX).

\subsection{Full Results}
 
\begin{table}[ttt!]
\centering
\caption{Negative log-likelihood on real datasets. $95\%$ confidence intervals (bootstrap) show. \colorbox{SpringGreen}{Green} indicates that a 2-Cats model was better than the best baseline (indicated in \colorbox{OrangeRed}{Red}). Underscores highlight ties with best model.}
\scriptsize
\begin{tabular}{lrrrr}
\toprule
& {Boston} & {INTC-MSFT} & {GOOG-FB} \\
\midrule
Par       & \underline{$-0.16 \pm 0.09$} & $-0.05 \pm 0.07$ & $-0.88 \pm 0.08$ \\
Bern      & $-0.09 \pm 0.06$ & $-0.03 \pm 0.05$ & $-0.70 \pm 0.04$ \\
PBern     & $-0.11 \pm 0.07$ & $-0.02 \pm 0.06$ & $-0.62 \pm 0.04$ \\
PSPL1     & $-0.06 \pm 0.14$ & $-0.02 \pm 0.06$ & $-0.72 \pm 0.14$ \\
PSPL2     & $-0.05 \pm 0.14$ & $-0.02 \pm 0.06$ & $-0.83 \pm 0.07$ \\
TTL0         & \underline{$-0.16 \pm 0.08$} & $-0.06 \pm 0.06$ & $-1.00 \pm 0.07$ \\
TTL1      & \underline{$-0.18 \pm 0.09$} & $-0.06 \pm 0.06$ & $-1.00 \pm 0.08$ \\
TTL2      & \underline{$-0.18 \pm 0.09$} & $-0.06 \pm 0.06$ & $-0.71 \pm 0.16$ \\
TLL2    & \underline{$-0.15 \pm 0.09$} & $-0.04 \pm 0.05$ & $-0.95 \pm 0.13$ \\
MR        & $-0.07 \pm 0.06$ & $-0.01 \pm 0.05$ & $-0.84 \pm 0.07$ \\
Probit      & $-0.10 \pm 0.06$ & $-0.03 \pm 0.07$ & $-0.92 \pm 0.06$ \\
\midrule
ACNet          & \underline{$-0.28 \pm 0.11$} & \cellcolor{OrangeRed}\underline{$-0.18 \pm 0.08$} & $-0.91 \pm 0.12$ \\
GEN-AC         & \cellcolor{OrangeRed}\underline{$-0.29 \pm 0.11$} & \underline{$-0.17 \pm 0.07$} & $-0.75 \pm 0.10$ \\
NL             & \underline{$-0.28 \pm 0.16$} & \underline{$-0.16 \pm 0.08$} & \cellcolor{OrangeRed}$-1.09 \pm 0.13$ \\
IGC            & $0.07  \pm 0.08$ & $0.14  \pm 0.04$ & $-0.30 \pm 0.02$ \\
\midrule
2-Cats-FLEX-L      & $0.38  \pm 0.71$ & $0.19  \pm 0.08$ & $-1.15 \pm 0.07$ \\
2-Cats-FLEX-G      & $0.50  \pm 0.43$ & $0.23  \pm 0.09$ & $-1.01 \pm 0.07$ \\
\midrule
2-Cats-G      & \cellcolor{SpringGreen}\underline{$-0.30 \pm 0.09$} & \cellcolor{SpringGreen}\underline{$-0.21 \pm 0.05$} & \cellcolor{SpringGreen}\underline{$-1.75 \pm 0.04$} \\
2-Cats-L      & \cellcolor{SpringGreen}\underline{$-0.27 \pm 0.07$} & $-0.07 \pm 0.06$ & \cellcolor{SpringGreen}\underline{$-1.65  \pm 0.08$} \\
\bottomrule
\end{tabular}
\label{tab:appn}
\end{table}
In Table~\ref{tab:appn}, we show the results for these models on real datasets. The definitions of the models are as follows: (2-Cats-P Gaus.) A parametric mixture of 10 Gaussian Copula densities.; (2-Cats-P Frank) Similar to the above, but is a mixture of 10 Frank Copula densities; (2-Cats-FLEX-G) A Flexible model version where the final activation layer is Bivariate Gaussian CDF; (2-Cats-FLEX-L) A Flexible version of the model where the final activation layer is Bivariate Logistic CDF.

Models were trained with the same hyperparameters and network definition as the ones in our main text.

\section{Appendix Baseline Source Code} \label{appn:code}

\begin{table}[ttt!]
\centering
\caption{Source Code Used as Baselines}
\scriptsize
\begin{tabular}{llr}
\toprule
& & Link  \\
\midrule
\multirow{11}{*}{\rotatebox[origin=c]{90}{{\em Non-Deep Learn}}} & Par & https://cran.r-project.org/web/packages/VineCopula/index.html \\
&Bern      & https://cran.r-project.org/web/packages/kdecopula/index.html \\
&PBern     & https://rdrr.io/cran/penRvine/ \\
&PSPL1     & https://rdrr.io/cran/penRvine/ \\
&PSPL2     & https://rdrr.io/cran/penRvine/ \\
&TTL0         & https://cran.r-project.org/web/packages/kdecopula/index.html \\
&TTL1      &https://cran.r-project.org/web/packages/kdecopula/index.html \\
&TTL2      & https://cran.r-project.org/web/packages/kdecopula/index.html \\
&TLL2    & https://cran.r-project.org/web/packages/kdecopula/index.html \\
&MR        & https://cran.r-project.org/web/packages/kdecopula/index.html \\
&Probit      &https://cran.r-project.org/web/packages/kdecopula/index.html\\
\midrule
\multirow{3}{*}{\rotatebox[origin=c]{90}{{\em Deep Learn}}}& ACNet          & https://github.com/lingchunkai/ACNet\\
&GEN-AC         & https://github.com/yutingng/gen-AC\\
&NL             & https://github.com/pawelc/NeuralLikelihoods0\\
&IGC            & https://github.com/TimCJanke/igc\\
\bottomrule
\end{tabular}
\label{tab:code}
\end{table}
The links for baseline source code are in Table~\ref{tab:code}.


\end{document}